\documentclass{article}

\usepackage{microtype}
\usepackage{graphicx}
\usepackage{subcaption}
\usepackage{booktabs} % for professional tables
\usepackage{hyperref}

\usepackage[preprint]{icml2026}

\usepackage{amsmath}
\usepackage{amssymb}
\usepackage{mathtools}
\usepackage{amsthm}

\usepackage{multirow}
\usepackage[table]{xcolor}

\usepackage{tcolorbox}

\usepackage[capitalize,noabbrev]{cleveref}

\theoremstyle{plain}
\newtheorem{theorem}{Theorem}[section]
\newtheorem{proposition}[theorem]{Proposition}

\newtheorem{corollary}[theorem]{Corollary}
\theoremstyle{definition}

\theoremstyle{remark}
\newtheorem{remark}[theorem]{Remark}
\icmlsetsymbol{it}{$\dagger$}
\usepackage[textsize=tiny]{todonotes}

\icmltitlerunning{Self-Hinting Language Models Enhance Reinforcement Learning}

\begin{document}

\twocolumn[
  \icmltitle{{Self-Hinting Language Models Enhance Reinforcement Learning}
}



  \icmlsetsymbol{equal}{*}

  \begin{icmlauthorlist}
    \icmlauthor{Baohao Liao}{equal,it,comp,yyy}
    \icmlauthor{Hanze Dong}{equal,comp}
    \icmlauthor{Xinxing Xu}{comp}
        \icmlauthor{Christof Monz}{yyy}
            \icmlauthor{Jiang Bian}{comp}
  \end{icmlauthorlist}

  \icmlaffiliation{comp}{Microsoft Research}
  \icmlaffiliation{yyy}{Language Technology Lab, University of Amsterdam}

  \icmlcorrespondingauthor{Hanze Dong}{hanzedong@microsoft.com}

  \icmlkeywords{Machine Learning, ICML}

  \vskip 0.3in
]



%\printAffiliationsAndNotice{}  % no special notice (required even if empty)
\printAffiliationsAndNotice{\icmlEqualContribution $^\dagger$This work was done during an internship at Microsoft.}

\begin{abstract}
Group Relative Policy Optimization (GRPO) has recently emerged as a practical recipe for aligning large language models with verifiable objectives. However, under sparse terminal rewards, GRPO often stalls because rollouts within a group frequently receive identical rewards, causing relative advantages to collapse and updates to vanish. 
We propose \emph{self-hint aligned GRPO with privileged supervision} (\textsc{SAGE}),
an on-policy reinforcement learning framework that injects privileged hints during training to reshape the rollout distribution under the \emph{same} terminal verifier reward.
For each prompt $x$, the model samples a compact hint $h$ (e.g., a plan or decomposition) and then generates a solution $\tau$ conditioned on $(x,h)$.
Crucially, the task reward $R(x,\tau)$ is unchanged; hints only increase within-group outcome diversity under finite sampling, preventing GRPO advantages from collapsing under sparse rewards.
At test time, we set $h=\varnothing$ and deploy the no-hint policy without any privileged information.
Moreover, sampling diverse self-hints serves as an adaptive curriculum that tracks the learner's bottlenecks more effectively than fixed hints from an initial policy or a stronger external model. Experiments over 6 benchmarks with 3 LLMs show that SAGE consistently outperforms GRPO, on average +2.0 on Llama-3.2-3B-Instruct, +1.2 on Qwen2.5-7B-Instruct and +1.3 on Qwen3-4B-Instruct. The code is available at \href{https://github.com/BaohaoLiao/SAGE}{https://github.com/BaohaoLiao/SAGE}.
\end{abstract}

\begin{figure}[t]
    \centering
    \includegraphics[width=0.98\linewidth]{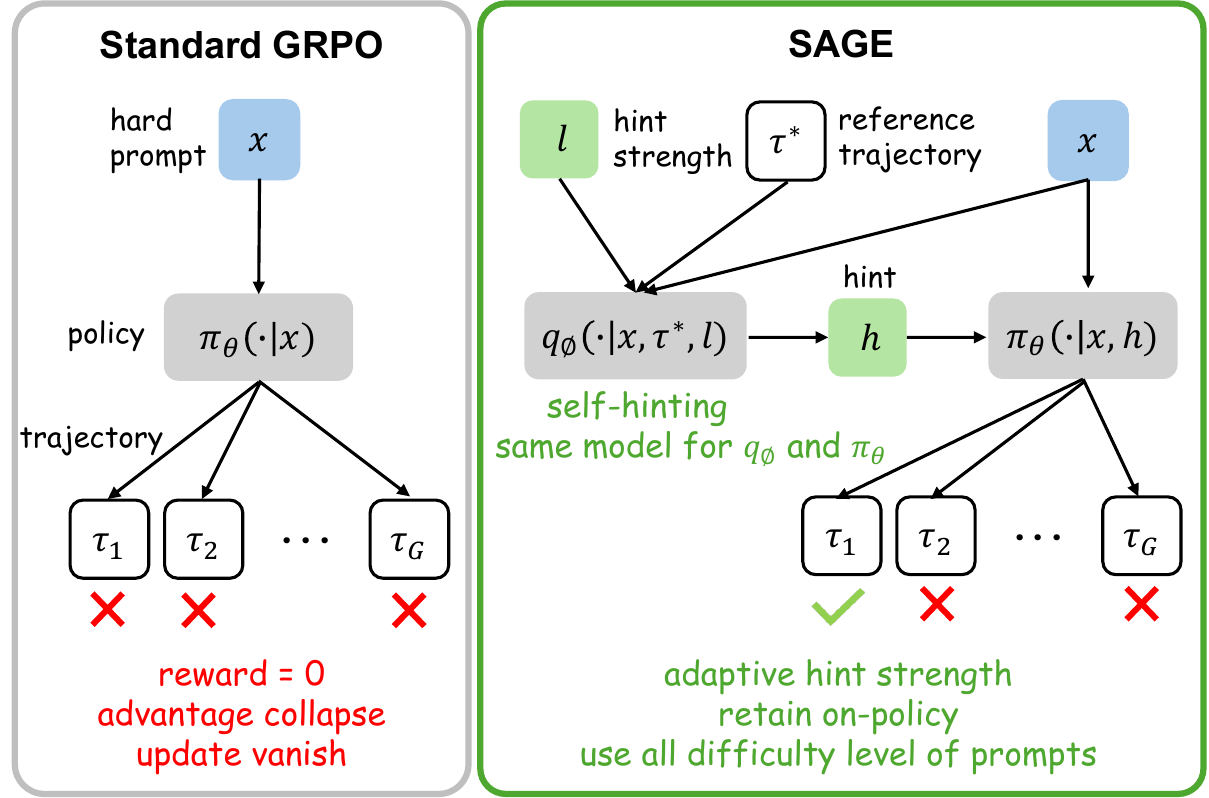}
    \caption{An overview of our proposed method, SAGE. When an LLM can't sample any correct trajectory for a prompt, the LLM \textbf{self-generates} hint from the reference solution of the prompt. The hint is then used together with the difficult prompt as input to the LLM, avoiding advantage collapse and ensuring the sampling of correct trajectories to update the policy model.}
    \label{fig:overview}
\end{figure}

\section{Introduction}

Reinforcement learning (RL) has become a core tool for training and aligning large language models (LLMs), particularly when supervision is most naturally expressed via verifiable objectives such as exact-match correctness, unit tests, or automated checkers \citep{ouyang2022training,schulman2017proximal,deepseekai2025deepseekr1incentivizingreasoningcapability}. In this setting, the objective is straightforward: maximization the expected reward over prompts, yet optimization can be fragile: with finite sampling, policy-gradient estimators may exhibit high variance and can even become degenerate on hard prompts.

A salient example arises with Group Relative Policy Optimization (GRPO) \citep{shao2024deepseekmath} under sparse terminal rewards. GRPO centers (and often standardizes) rewards within each rollout group, relying on within-group outcome differences to produce a nonzero update. With a $0/1$ verifier, difficult prompts frequently yield groups where all rollouts receive the same reward (typically all zeros). In that case, the group-centered advantages collapse and the minibatch policy-gradient estimate becomes identically zero. Importantly, this is a \textbf{finite-sample pathology}: the underlying expected objective needs not be flat, but the estimator provides no learning signal for many prompts.

Existing remedies largely modify data collection. A common baseline is to skip uninformative updates (e.g., degenerate groups) and resample prompts, which improves performance but implicitly biases training toward easier prompts \citep{yu2025dapo,xiong2025minimalist}. More systematic approaches include adaptive sampling or curriculum-style scheduling to allocate more rollouts to difficult prompts \citep{yao2025optimizing,xiong2025reinforce,li2025knapsack,zhang2025improving}, as well as leveraging offline data or externally generated candidates (e.g., from stronger models) to bootstrap learning \citep{zhang2025stephint,yan2025learning,zhang2025scaf}. While effective, these strategies can either skew the training distribution or introduce context/distribution mismatch that must be handled carefully.

We propose \textsc{SAGE} (Self-hint Aligned GRPO with Privileged Supervision), a complementary approach based on \emph{privileged hinting}. During training, we provide an additional hint $h$, a lossy compression of a reference solution $\tau^\star$, and roll out from the hint-conditioned policy $\pi_\theta(\cdot\mid x,h)$. Hints only reshape the rollout distribution to increase the probability of observing mixed outcomes within a finite group.  At test time, we deploy the no-hint policy. We refer to hints generated by the policy itself as self-hints, and to the procedure of generating such hints as self-hinting.

This degeneracy can be made explicit. Let $p_\theta(x)$ be the no-hint success probability and $G$ the group size. The probability that a rollout group contains mixed outcomes is
$
1-(1-p_\theta(x))^G-p_\theta(x)^G \approx Gp_\theta(x),
$
so updates vanish whenever $Gp_\theta(x)\ll 1$. Hinting is useful precisely when it increases the effective success probability so that mixed-outcome groups become common for the same $G$.

\textbf{Contributions.}
(i) We introduce {SAGE} as shown in Figure \ref{fig:overview}, an on-policy RL framework that conditions rollouts on privileged hints during training while keeping the task reward unchanged and removing hints at test time.
(ii) We develop a policy-dependent hint-strength scheduler, that activates hints only when within-group rewards collapse, yielding an automatic curriculum.
(iii) We propose an online self-hinting scheme that periodically refreshes the hint distribution during training to maintain calibration to the learner, avoiding overly weak/overly strong fixed hints.
(iv) We provide analysis that characterizes GRPO collapse as a gate-opening probability under Bernoulli rewards and empirically validate that \textsc{SAGE} improves sample efficiency and final accuracy on challenging reasoning benchmarks.

\begin{figure}[t]
    \centering
    \includegraphics[width=0.98\linewidth]{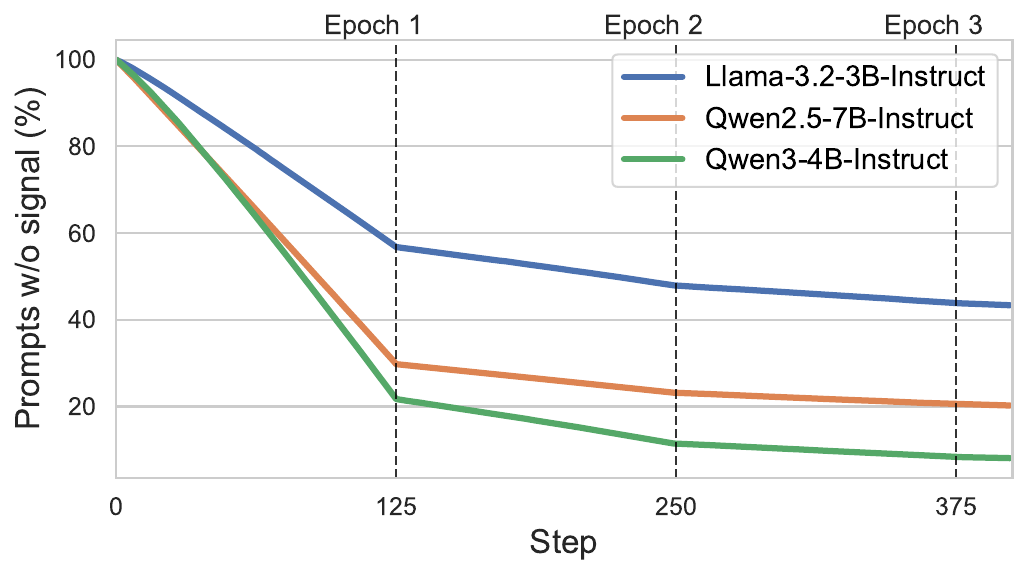}
    \caption{The percentage of prompts whose correct trajectories have NEVER been sampled w.r.t. the training step. Here we train on 64k prompts, and sample 8 traces per prompt per step. A large number of prompts is wasted during RL, especially for a weaker LLM, since they don't offer any signal for training.}
    \label{fig:pass_k_0_during_training}
\end{figure}

\section{RL with Privileged Hinting}
\label{sec:hint_grpo}

Vanilla GRPO works well when a prompt $x$ yields occasional positive rollouts.
In hard regimes, groups often receive identical rewards, collapsing within-group advantages and stalling learning.
We address this failure mode by injecting \emph{privileged hints} during training that keep the reward unchanged while reshaping the rollout distribution to surface informative trajectories under finite sampling.

\subsection{Setup and the GRPO stall}
Let $x \sim \mathcal D$ be a prompt.
A policy $\pi_\theta$ generates a trajectory $\tau=(y_1,\dots,y_T)$, written as $\tau\sim\pi_\theta(\cdot\mid x)$.
We use a binary reward
\(
R(x,\tau)\in\{0,1\}.
\)
Define the success probability
\begin{equation}
p_\theta(x)=\Pr_{\tau\sim\pi_\theta(\cdot\mid x)}[R(x,\tau)=1].
\end{equation}
GRPO implementations often standardize groupwise advantages by the within-group standard deviation.
For a group of $G$ rollouts $\{\tau_i\}_{i=1}^G$, let
\begin{equation*}
R_i=R(x,\tau_i),\, \bar R=\tfrac{1}{G}\sum_{i=1}^G R_i,
\,
s^2=\tfrac{1}{G}\sum_{i=1}^G (R_i-\bar R)^2,
\end{equation*}
and define standardized advantages
\begin{equation}
 A_i=\frac{R_i-\bar R}{s+\epsilon},
\label{eq:std_adv}
\end{equation}
where $\epsilon\ge 0$ is a numerical stabilizer and $R_i\in\{0,1\}$.
When $p_\theta(x)$ is tiny, a group is often all-zero and $A_i=0$ for all $i$. The chance of a non-degenerate group is
\begin{equation*}
1-(1-p_\theta(x))^G-p_\theta(x)^G \approx 1-(1-p_\theta(x))^G \approx Gp_\theta(x),
\label{eq:rare_signal_simple}
\end{equation*}
so training stalls whenever $Gp_\theta(x)\ll 1$ on most prompts. Figure~\ref{fig:pass_k_0_during_training} illustrates this phenomenon in practice: for many hard prompts, correct trajectories are never sampled for a long stretch of training, yielding no learning signal.

\subsection{Privileged hinting as sampling}
When a reference trajectory $\tau^\star$ is available during training, we generate a hint $h$ as a lossy compression of $\tau^\star$.
The hint is appended to the prompt as additional context.

\textbf{Hint strength and the no-hint case.}
We control hint informativeness with a discrete strength level $\ell\in\{0,1,\dots,L\}$, where larger $\ell$ indicates more information about the reference trajectory $\tau^\star$.
We sample
\begin{equation}
\ell \sim p(\ell),
\qquad
h \sim q(h\mid x,\tau^\star,\ell),
\label{eq:hint_sampling_simple}
\end{equation}
where $\ell=0$ corresponds to the no-hint setting and
$q(h\mid x,\tau^\star,0)=\delta_{\varnothing}(h)$ (i.e., $h=\varnothing$ deterministically).

\textbf{Policy-dependent scheduling of $\ell$.}
Hints should be used only when a prompt provides no learning signal.
We let the sampling of $\ell$ depend on the policy through a simple statistic, such as a collapse indicator
$
c(x)=\mathbb I\!\left[\mathrm{Var}\big(\{R(x,\tau_i)\}_{i=1}^G\big)=0\right],
$
computed from a small probe group under the policy model $\pi_{\theta}$.
A minimal scheduler is
\begin{equation}
p(\ell\mid x)=
\begin{cases}
\delta_0, & c(x)=0,\\
 p(\ell), & c(x)=1,
\end{cases}
\label{eq:ell_scheduler_simple}
\end{equation}
so $\ell>0$ is activated only when the group collapses.

With a hint, we sample from $\pi_\theta(\cdot\mid x,h)$.
The success rate increases:
$
p_\theta^{(\ell)}(x)=\Pr_{\tau\sim\pi_\theta(\cdot\mid x,h)}[R(x,\tau)=1],
h\sim q(\cdot\mid x,\tau^\star,\ell).
$
Hinting is useful when it raises $p_\theta^{(\ell)}(x)$ enough that $Gp_\theta^{(\ell)}(x)$ is no longer tiny, so non-degenerate groups become common and RL receives updates.

\subsection{GRPO with hints and the final loss}
Given $x$, sample $\ell$ and $h$, then draw a group of rollouts
\begin{equation*}
\tau_i \sim \pi_\theta(\cdot\mid x,h),
\qquad
R_i=R(x,\tau_i),
\quad i=1,\dots,G.
\end{equation*}
Compute $A_i$ with Eq.~\eqref{eq:std_adv}.
The loss conditioned on hints is
\begin{equation}
\begin{aligned}
\mathcal L(\theta)=
-\mathbb E\!\left[
\tfrac{1}{G}\sum_{i=1}^G
A_i \sum_{t=1}^{T_i}\log \pi_\theta(y_{i,t}\mid x,h,y_{i,<t})
\right]
\\+\beta\,\mathbb E\!\left[\mathrm{KL}\big(\pi_\theta(\cdot\mid x,h)\,\|\,\pi_{\mathrm{ref}}(\cdot\mid x,h)\big)\right],
\end{aligned}
\label{eq:final_loss_simple}
\end{equation}
where the expectation is over $x\sim\mathcal D$, the hint sampling from Eq.~\eqref{eq:hint_sampling_simple}, and rollouts from $\pi_\theta(\cdot\mid x,h)$.
At test time we set $\ell=0$, $h=\varnothing$ and run the  policy
$\pi_\theta(\cdot\mid x,\varnothing)\equiv\pi_\theta(\cdot\mid x)$.

\textbf{Summary.}
Sparse rewards can cause GRPO to stall because, for many prompts $x$, finite groups contain no positive samples and advantages collapse.
Privileged hinting fixes this by changing the rollout distribution for such prompts while keeping the reward unchanged.
A policy-dependent scheduler activates hints only when groups collapse, yielding an automatic curriculum.
Training remains on-policy since rollouts are drawn from $\pi_\theta(\cdot\mid x,h)$.
Deployment uses $\ell=0$ and requires no hints or privileged information.

\section{Analysis}

\subsection{Standardized GRPO as a gated update objective}
\label{sec:analysis_sage}

Fix a context $(x,h)$ and draw $G$ rollouts with rewards $R_i\in\{0,1\}$.
Let $\bar R=\tfrac{1}{G}\sum_i R_i$ and $s^2=\tfrac{1}{G}\sum_i (R_i-\bar R)^2$.
Standardized GRPO uses $
A_i=\frac{R_i-\bar R}{s+\epsilon}.$

\begin{corollary}[Signal energy equals a gate probability]
\label{corl:gate_energy}
Define the advantage energy
$
E \coloneqq \tfrac{1}{G}\sum_{i=1}^G A_i^2.
$
If $\epsilon>0$, 
\begin{equation}
E = \frac{s^2}{(s+\epsilon)^2}\in[0,1],
\label{eq:energy_eps}
\end{equation}
which is monotone in $s$ and still collapses to $0$ when $s=0$.
\end{corollary}
% \begin{proof}
% For $\epsilon>0$,
% $
% E=\tfrac{1}{G}\sum_i \tfrac{(R_i-\bar R)^2}{(s+\epsilon)^2}
% = \tfrac{s^2}{(s+\epsilon)^2}.
% $
% \end{proof}

For GRPO, the prompt-level update magnitude is dominated by whether the group is
\emph{non-degenerate} ($s>0$). In other words, training behaves like a gated procedure that updates
only when the rollout group contains mixed outcomes.

\begin{proposition}
[Gate opening probability under Bernoulli rewards]
\label{cor:gate_prob}
Let $p_\theta(x,h)=\Pr[R(x,\tau)=1 \mid x,h]$.
Then
\begin{equation}
\Pr[s>0 \mid x,h]
=
1-(1-p_\theta(x,h))^G-p_\theta(x,h)^G,
\label{eq:gate_prob}
\end{equation}
$\Pr[s>0\mid x,h]$ is maximized at $p_\theta=\tfrac{1}{2}$.
In the sparse regime $p_\theta(x,h)\ll 1$, $\Pr[s>0\mid x,h]\approx G\,p_\theta(x,h)$.
\end{proposition}

Thus, \textsc{SAGE} should choose the hint strength to move hard prompts out of the regime where $Gp_\theta\ll 1$, and avoid overly strong hints that push $p_\theta\approx 1$ to close the gate again.

\begin{figure*}[th]
    \centering
    \begin{subfigure}{0.49\textwidth}
        \centering
        \includegraphics[width=\linewidth]{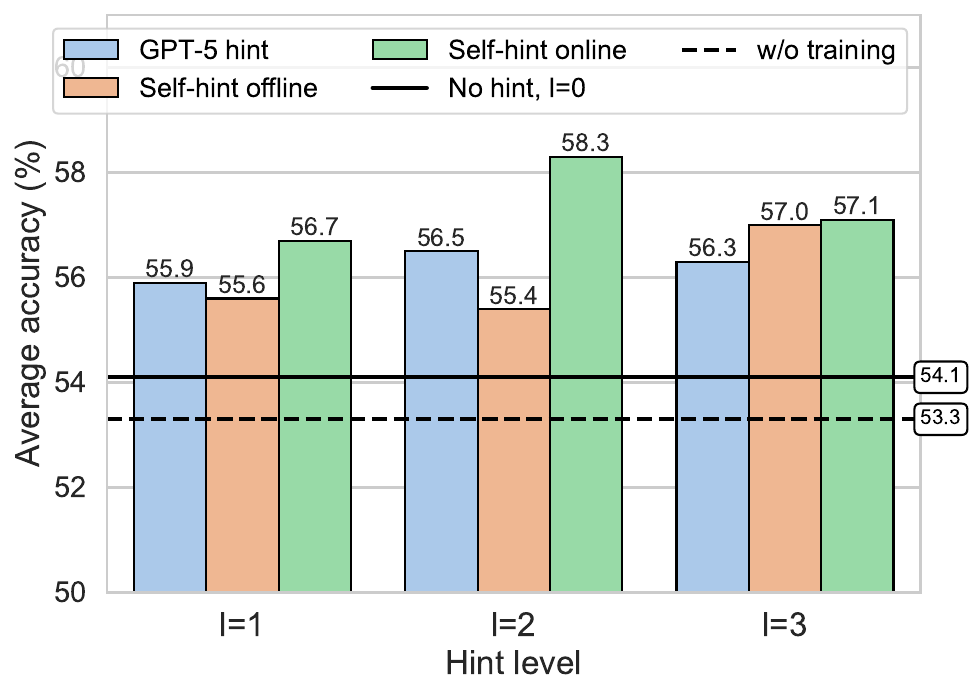}
        \label{fig:poc_hints_bar}
        \vspace{-6mm}
    \end{subfigure}
    \begin{subfigure}{0.49\textwidth}
        \centering
        \includegraphics[width=\linewidth]{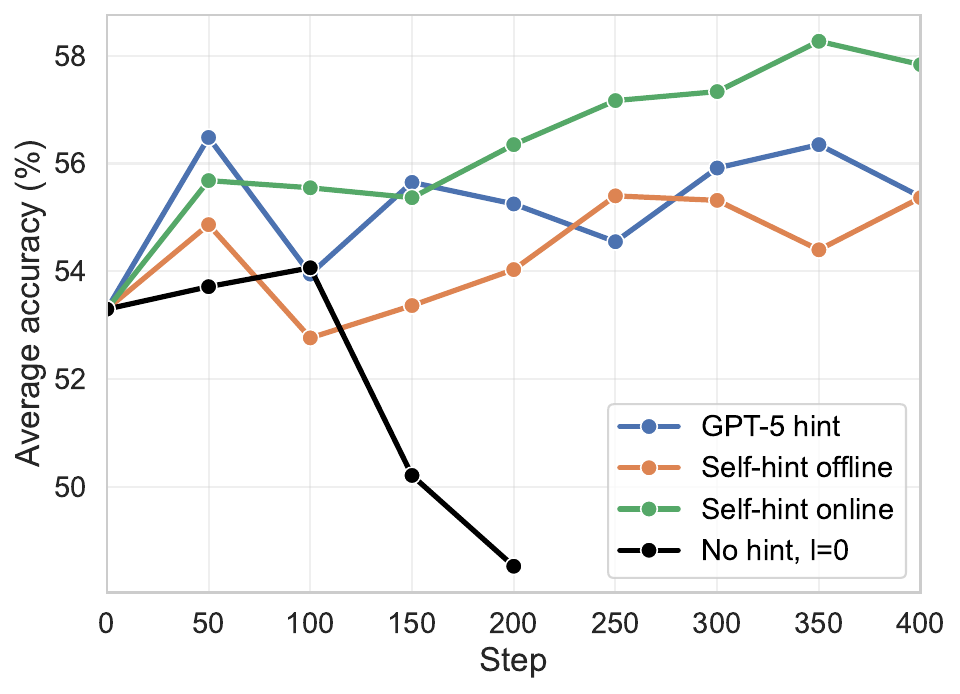}
        \label{fig:poc_hints_line}
        \vspace{-6mm}
  \end{subfigure}
  \caption{Average accuracy on Qwen3-4B-Instruct over 6 benchmarks. The 4.5k training prompts here are extremely hard, whose correct trajectories have never been sampled during training as Figure \ref{fig:pass_k_0_during_training}. The number of rollouts per prompt per step here is set to 32 to encourage exploration. \textbf{Left}: Performance on various hints. Training without hint only slightly improves the performance, since the reward signal from the hard prompts is sparse. However, training with any hint boosts the performance. Among all methods, online self-hinting consistently achieves the best performance across different hint levels. \textbf{Right}: Average accuracy w.r.t. the training steps for hint level $l=2$. Training without any hint even degrades the performance as the training goes, since the reward signal is too sparse, making it overfit to a few solvable prompts. However, online self-hinting boosts the performance steadily. Refer to Table \ref{tab:detailed_number_for_hint_levels} for detailed number.}
  \label{fig:poc_hints}
\end{figure*}

\begin{proposition}[Optimal hint distribution is policy-dependent]
\label{prop:q_dep_theta}
Fix a prompt $x$ and group size $G\ge 2$.
Let $p_\theta(x,h)=\Pr_{\tau\sim\pi_\theta(\cdot\mid x,h)}[R(x,\tau)=1]$ and define
\begin{equation}
u(p)\coloneqq 1-(1-p)^G-p^G.
\label{eq:gate_u}
\end{equation}
Under Bernoulli rewards, $u(p_\theta(x,h))=\Pr[s>0\mid x,h]$.

For any distribution $q(\cdot\mid x)$, define the expected probability
\begin{equation}
J_x(\theta,q)\coloneqq \mathbb E_{h\sim q(\cdot\mid x)}\!\left[u\!\left(p_\theta(x,h)\right)\right].
\label{eq:J_theta_q}
\end{equation}
Then $u$ is symmetric and strictly concave on $[0,1]$, and is maximized at $p=\tfrac12$.
Consequently, any maximizer
\begin{equation}
q_\theta^\star(\cdot\mid x)\in\arg\max_q J_x(\theta,q)
\end{equation}
must place its mass on \emph{calibrating hints} that make $p_\theta\approx\tfrac12$.

In general, the set of calibrating hints depends on $\theta$.
Unless $p_\theta(x,h)$ is invariant in $\theta$ for $q$-almost all $h$, a fixed $q$ cannot remain (near) optimal for $J_x(\theta,q)$ throughout training.
Updating $q$ online therefore reduces this gate-mismatch and increases the frequency of non-degenerate GRPO updates.
\end{proposition}

\begin{remark}[Why not sample many hints per prompt.]
Let $u(p)=1-(1-p)^G-p^G$ denote the gate probability in \eqref{eq:gate_prob}.
For $G\ge 2$, $u$ is concave on $[0,1]$ since
$
u''(p)=-G(G-1)\big(p^{G-2}+(1-p)^{G-2}\big)\le 0.
$
Therefore, for any hint distribution $h\sim q$,
\begin{equation}
\mathbb E_{h\sim q}\big[u(p_\theta(x,h))\big]
\le
u\!\left(\mathbb E_{h\sim q}[p_\theta(x,h)]\right).
\label{eq:jensen_gate}
\end{equation}
At a fixed mean success rate, additional randomness across hints can only reduce the
expected frequency of non-degenerate groups.
Motivated by \eqref{eq:jensen_gate}, we sample a single hint realization per prompt per epoch and spend
compute on $G$ or better strength scheduling.
\end{remark}

Overall, standardized GRPO turns sparse-reward learning into maximizing a gate probability.
The proposed method should operationalize this by (i) ensuring on-policy conditioning on $h$,
(ii) scheduling $\ell$ when the gate is closed,
 (iii) updating the self-hint generator online to keep $p_\theta$ calibrated.

\section{Design of SAGE}
\label{sec:design_sage}

% \textsc{SAGE} is designed to revive GRPO in sparse-reward regimes by
% introducing privileged hints during training.
% The core principle is simple: \emph{hints change how we sample rollouts, not how we reward them}.
% Concretely, we always evaluate trajectories with the same reward $R(x,\tau)$, while hints
% tilt the rollout distribution so that finite groups contain non-identical outcomes and GRPO
% advantages do not collapse.
% This section describes four implementation choices: on-policy training, policy-dependent hint-strength
% scheduling, online self-hint generation, and the full algorithm.

\subsection{On-policy training}
\label{sec:onpolicy_sage}

\textbf{Why hints must be in the conditioning context.}
\textsc{SAGE} appends a hint $h$ to the prompt and samples rollouts from the hint-conditioned policy
$\pi_\theta(\cdot\mid x,h)$.
This is not merely a modeling choice: it is what keeps training \emph{on-policy} for the augmented
context.
The loss is
$\sum_t \log \pi_\theta(y_t\mid x,h,y_{<t})$.
If one instead samples $\tau\sim\pi_\theta(\cdot\mid x,h)$ but evaluates
$\log \pi_\theta(\tau\mid x)$ (i.e., dropping $h$ inside the log-prob),
the update no longer corresponds to the gradient of any on-policy objective under the sampling
process.
In practice this mismatch behaves like an off-policy update and is markedly less stable under sparse
rewards. We also include a controlled ablation (Sec.~\ref{sec:dis}) that keeps the \emph{sampling} process identical (roll out with
hint) but changes the conditioning.

\subsection{Online self-hinting}
\label{sec:online_hint}

In the algorithm, we can produce privileged hints in two ways:
(1) \textbf{Offline hints (fixed).}
    A fixed hint generator (e.g., extracted once from $\tau^\star$) is simple, but it does not adapt to the learner and can become miscalibrated over training.
(2) \textbf{Online hints.}
    We periodically refresh the hint generator using a copy of the current policy $\pi_\theta$.

In SAGE, hint generation is \emph{online}. We implement $q_{\phi}(h\mid x,\tau^\star,\ell)$ by prompting the policy $\pi_{\theta}$ to produce a \emph{procedure-only} plan aligned with the reference trajectory.
%This design keeps hints calibrated to the model's evolving capabilities: as the policy improves, the same hint level $\ell$ naturally yields more compact and higher-level guidance, while persistent failure modes trigger stronger or more detailed hints through the scheduler. 

We evaluate three variants: (1) \textbf{Fixed privileged hints:} $q_{\phi}$ is derived from $\pi_{\theta_0}$ and frozen after initialization.
(2) \textbf{Online privileged hints (\textsc{SAGE}):} $q_{\phi}$ is derived from $\pi_{\theta}$ and refreshed during training.
(3) \textbf{External teacher hints:} $q_{\phi}$ is produced by a stronger frozen model, when available.

By Figure~\ref{fig:poc_hints}, adding hints improves performance across hint levels compared with no hint, consistent with hints increasing the chance of sampling informative trajectories under sparse rewards.
Besides, online self-hinting consistently performs best, indicating that continually refreshed self-hints are better calibrated to the learner than fixed hints.

\begin{algorithm*}[t]
\caption{\textsc{SAGE}  / \textsc{SAGE-light}: Self-hint Aligned GRPO with Privileged Supervision}
\label{alg:priv_hint_grpo}
\begin{algorithmic}[1]
\small
\REQUIRE Training set $\mathcal{D}=\{(x,\tau^\star)\}$, policy model $\pi_\theta$, group size $G$,
KL weight $\beta$, stabilizer $\epsilon>0$, max hint level $L$, reference policy $\pi_{\mathrm{ref}}$, hint generator $q_\phi(h\mid x,\tau^\star,\ell)$ based on $\pi_\theta$, threshold $\alpha$ (\textsc{SAGE-light} only).

\STATE Initialize policy parameters $\theta$; initialize per-prompt level map $\ell(x)\leftarrow 0$ for all $x\in\mathcal{D}$.
\STATE Repeat for epochs:
\STATE \quad \textbf{for} each minibatch $\{(x_b,\tau_b^\star)\}_{b=1}^B$:
\STATE \qquad \textbf{for} $b=1,\dots,B$ \textbf{do}
\STATE \qquad\quad \textbf{if} \textsc{\textsc{SAGE-light}} \textbf{then}
\STATE \qquad\qquad \textbf{if} epoch $>1$ \textbf{and} $\bar{R}_b<\alpha$ \textbf{then} $\ell(x_b)\leftarrow \min\{\ell(x_b)+1,L\}$ 
\STATE \qquad\qquad Sample $h_b\sim q_\phi(h\mid x_b,\tau_b^\star,\ell(x_b))$
\STATE \qquad\quad \textbf{else} \hfill (\textsc{SAGE})
\STATE \qquad\qquad \textbf{for} $\ell=0,\dots,L$ \textbf{do}
\STATE \qquad\qquad\quad Sample $\tilde{h}_b\sim q_\phi(h\mid x_b,\tau_b^\star,\ell)$
\STATE \qquad\qquad\quad Sample $\tilde{\tau}_{b,i}\sim \pi_\theta(\cdot\mid x_b,\tilde{h}_b)$ for $i=1,\dots,G$ and compute $\tilde{R}_{b,i}\leftarrow R(x_b,\tilde{\tau}_{b,i})$
\STATE \qquad\qquad\quad \textbf{if} $\sum_{i=1}^G \tilde{R}_{b,i}>0$ \textbf{or} $\ell=L$ \textbf{then}
\STATE \qquad\qquad\qquad $h_b\leftarrow \tilde{h}_b,\; \tau_{b,i}\leftarrow \tilde{\tau}_{b,i},\; R_{b,i}\leftarrow \tilde{R}_{b,i}$; \textbf{break}
\STATE \qquad\quad {(If \textsc{{SAGE-light}})} Sample $\tau_{b,i}\sim \pi_\theta(\cdot\mid x_b,h_b)$ for $i=1,\dots,G$ and compute $R_{b,i}\leftarrow R(x_b,\tau_{b,i})$.
\STATE \qquad\quad Compute $\bar{R}_b \leftarrow \tfrac{1}{G}\sum_{i=1}^G R_{b,i}$,
$s_b \leftarrow \sqrt{\tfrac{1}{G}\sum_{i=1}^G (R_{b,i}-\bar{R}_b)^2}$,
$A_{b,i}\leftarrow \tfrac{R_{b,i}-\bar{R}_b}{s_b+\epsilon}$.
\STATE \qquad $\mathcal{L}_{\mathrm{pg}} \leftarrow -\tfrac{1}{BG}\sum_{b=1}^B\sum_{i=1}^G A_{b,i}\sum_{t=1}^{T_{b,i}}\log \pi_\theta\!\left(y_{b,i,t}\mid x_b,h_b,y_{b,i,<t}\right)$
\STATE \qquad $\mathcal{L}_{\mathrm{kl}} \leftarrow \tfrac{1}{B}\sum_{b=1}^B
\mathrm{KL}\!\left(\pi_\theta(\cdot\mid x_b,h_b)\,\|\,\pi_{\mathrm{ref}}(\cdot\mid x_b,h_b)\right)$
\STATE \qquad Update $\theta \leftarrow \theta - \eta\,\nabla_\theta\big(\mathcal{L}_{\mathrm{pg}}+\beta\,\mathcal{L}_{\mathrm{kl}}\big)$
\STATE \textbf{Deployment:} set $\ell=0$ so $h=\varnothing$, and run $\pi_\theta(\cdot\mid x)$.
\end{algorithmic}
\end{algorithm*}

\subsection{Policy-dependent scheduling}
\label{sec:sage_scheduling}

We control hint informativeness with a discrete strength variable $\ell\in\{0,1,\dots,L\}$,
where $\ell=0$ corresponds to the deployable no-hint setting.
The scheduler is \emph{policy-dependent} in the sense that it adapts $\ell$ using statistics collected
from recent rollouts under the policy $\pi_{\theta}$ (stop-gradient).
Intuitively, we increase $\ell$ only when the current policy provides insufficient learning signal on a
prompt.

\textbf{Scheme 1 (\textsc{SAGE-light}): epoch-level accuracy threshold.}
Let $\hat p_{t-1}(x)$ denote the empirical success rate of prompt $x$ measured in the previous epoch using rollouts from policy model with current hint level.
Given a target threshold $\alpha\in(0,1)$, we increase hint strength when the prompt is too hard:
\begin{equation}
\ell_t(x)=
\min\big\{\,\ell_{t-1}(x)+1,\; L\,\big\}
\quad \text{if } \hat p_{t-1}(x)<\alpha,
\label{eq:sched_acc}
\end{equation}
and otherwise keep $\ell_t(x)=\ell_{t-1}(x)$.
Thus, hints are activated only when success is consistently rare.

\textbf{Scheme 2 (\textsc{SAGE}): group-degeneracy trigger.}
GRPO requires within-group outcome differences to produce a nonzero update.
We therefore use a more local trigger based on whether a probe group contains any positive sample.
For a small probe group $\{\tau_i\}_{i=1}^G$ rolled out from $\pi_{\theta}(\cdot\mid x,h)$ at the current
strength $\ell_{t-1}(x)$, define
$
z(x)=\mathbb I\!\left[\sum_{i=1}^G R(x,\tau_i)=0\right],
$
i.e., $z(x)=1$ when the group has no positive rollouts.
We then increase hint strength only on such collapsed prompts:
\begin{equation}
\ell_t(x)=
\min\big\{\,\ell_{t-1}(x)+1,\; L\,\big\}
\quad \text{if } z(x)=1,
\label{eq:sched_nopos}
\end{equation}
and otherwise keep $\ell_t(x)=\ell_{t-1}(x)$.
This rule targets the specific finite-sample pathology of sparse rewards: when a group contains no
positives, standardized GRPO advantages collapse and the policy-gradient estimator vanishes.

\textbf{Discussion.}
\textsc{SAGE-light} (Scheme~1) is compute-efficient because it updates the hint strength only at the epoch level (no additional computation for rollouts), but it can react slowly to sudden reward collapse.
\textsc{SAGE}  (Scheme~2) is more reactive and directly targets GRPO's failure mode via a no-positives trigger, at the cost of additional probe rollouts.
We report results for both schemes, and find that the no-positives trigger typically recovers faster from stalled training on hard prompts and yields better performance.

Overall, Algorithm~\ref{alg:priv_hint_grpo} summarizes \textsc{SAGE} implementation.
Each prompt draws \emph{one} hint per epoch (given $\ell$), and all $G$ rollouts for that prompt share
the sam context.
This design reduces unnecessary variance from hint.

\begin{table*}[t]
    \caption{Accuracy on in-distribution and out-of-distribution tasks across three LLMs. The \textbf{best} and \underline{second-best} results are in bold and underlined, respectively. \textsc{SAGE}  and \textsc{SAGE-light} consistently outperform baselines on average across various LLMs.}
    \label{tab:main_results}
    \vspace{-2mm}
    \scriptsize
    \begin{center}
    \begin{tabular}{lccccccrcccr}
    \toprule
        \multirow{3}{*}{\textbf{Method}} & \multicolumn{7}{c}{\textbf{In-distribution}} & \multicolumn{4}{c}{\textbf{Out-of-distribution}} \\
        \cmidrule(lr){2-8} \cmidrule(lr){9-12} 
        & \textbf{AIME24 / 25} & \textbf{AMC23} & \textbf{MATH-500} & \textbf{Minerva} & \textbf{Olympiad} & \textbf{Avg.} & \textbf{$\Delta$} & \textbf{GPQA} & \textbf{MMLU-Pro} & \textbf{Avg.} & \textbf{$\Delta$} \\
    \midrule
        \textit{Llama-3.2-3B-Instruct} & 6.5 / 0.6 & 22.8 & 44.7 & 17.8 & 14.2 & 17.8 & 0 & 17.9 & 27.0 & 22.5 & 0 \\
        \cmidrule(lr){1-12}
        SFT & 0.4 / 0.6 & 9.5 & 26.9 & 5.1 & 6.5 & 8.2 & $-$9.6 & 11.6 & 18.8 & 15.2 & $-$7.3 \\
        GRPO & 6.7 / 0.8 & 29.5 & 52.1 & \underline{20.5} & \underline{21.8} & 21.9 & $+$4.1 & 26.3 & \underline{39.8} & 33.1 & $+$10.6 \\
        LUFFY & 4.4 / 0.4 & 18.6 & 38.9 & 14.3 & 11.9 & 14.7 & $-$3.1 & 16.0 & 26.7 & 21.4 & $-$1.1 \\
        Scaf-GRPO & 7.7 / \textbf{2.3} & 28.8 & 51.7 & 19.4 & 19.5 & 21.5 & $+$3.7 & 24.1 & 38.0 & 31.0& $+$8.5 \\
        \rowcolor{blue!8}
        \textsc{SAGE-light} & \underline{8.8} / \underline{1.9} & \underline{32.2} & \underline{54.1} & \textbf{20.8} & 20.1 & \underline{23.0} & \underline{$+$5.2} & \underline{26.8} & 39.6 & \underline{33.2} & \underline{$+$10.7} \\
        \rowcolor{orange!8}
        \textsc{SAGE}  & \textbf{9.2} / 0.8 & \textbf{34.7} & \textbf{56.3} & 20.1 & \textbf{22.0} & \textbf{23.9} & \textbf{$+$6.1} & \textbf{27.3} & \textbf{40.7} & \textbf{34.0} & \textbf{$+$11.5} \\
    \midrule
        \textit{Qwen2.5-7B-Instruct} & 13.8 / 6.7 & 53.4 & 75.7 & 38.1 & 39.2 & 37.8 & 0 & 37.1 & 56.4 & 46.7 & 0 \\
        \cmidrule(lr){1-12}
        SFT & 3.5 / 7.1 & 30.9 & 56.2 & 20.0 & 21.7 & 23.2 & $-$14.6 & 9.5 & 35.6 & 22.5 & $-$24.2 \\
        GRPO & 15.0 / \textbf{13.5} & 55.5 & 79.2 & 39.1 & 44.5 & 41.1 & $+$3.3 & 37.2 & 57.6 & 47.4 & $+$0.7 \\
        LUFFY & \textbf{17.1} / \textbf{13.5} & 55.2 & \textbf{81.3} & 39.0 & 44.2 & 41.7 & $+$3.9 & \textbf{38.1} & \underline{59.1} & \textbf{48.6} & \textbf{$+$1.9} \\
        Scaf-GRPO & 14.6 / \underline{12.7} & \underline{58.8} & 78.0 & \textbf{39.8} & 42.0 & 41.0 & $+$2.2 & 36.6 & 58.4 & 47.5 & $+$0.8 \\
        \rowcolor{blue!8}
        \textsc{SAGE-light} & \textbf{17.1} / 11.7 & 58.1 & 79.9 & 38.6 & \textbf{46.1} & \underline{41.9} & \underline{$+$4.1} & 36.6 & 58.8 & \underline{47.7} & \underline{$+$1.0} \\
        \rowcolor{orange!8}
        \textsc{SAGE}  & \underline{16.0} / 12.5 & \textbf{60.3} & \underline{80.0} & \underline{39.3} & \underline{45.9} & \textbf{42.3} & \textbf{$+$4.5} & \underline{38.0} & \textbf{59.3} & \textbf{48.6} & \textbf{$+$1.9} \\
    \midrule
        \textit{Qwen3-4B-Instruct} & 52.1 / 43.1 & 92.2 & 93.6 & 46.1 & 67.7 & 65.8 & 0 & \underline{57.6} & 70.9 & 64.3 & 0 \\
        \cmidrule(lr){1-12}
        SFT & 14.4 / 22.1 & 55.2 & 78.9 & 38.1 & 40.2 & 41.5 & $-$24.3 & 29.4 & 52.6 & 41.0 & $-$23.3 \\
        GRPO & 55.8 / 45.0 & \textbf{95.0} & 96.0 & \textbf{50.1} & 70.4 & 68.7 & $+$2.9 & 57.0 & 72.0 & \underline{64.5} & \underline{$+$0.2} \\
        LUFFY & 42.3 / 36.0 & 86.4 & 91.1 & 48.2 & 59.4 & 60.6 & $-$5.2 & 31.9 & 46.6 & 39.3 & $-$25.0 \\
        Scaf-GRPO & \textbf{59.8} / 45.2 & 92.2 & \underline{95.1} & 48.9 & 69.8 & 68.5 & $+$2.7 & 54.3 & \underline{72.1} & 63.2 & $-$1.1 \\
        \rowcolor{blue!8}
        \textsc{SAGE-light} & \underline{59.2} / \underline{47.1} & 92.2 & \underline{95.1} & \underline{49.5} & \underline{70.5} & \underline{68.9} & \underline{$+$3.1} & 57.1 & 72.0 & \underline{64.5} & \underline{$+$0.2} \\
        \rowcolor{orange!8}
        \textsc{SAGE}  & 58.1 / \textbf{52.1} & \underline{94.2} & \textbf{95.4} & 49.2 & \textbf{71.2} & \textbf{70.0} & \textbf{$+$4.2} & \textbf{57.8} & \textbf{72.5} & \textbf{65.2} & \textbf{$+$0.9} \\
    \bottomrule
    \end{tabular}
    \end{center}
    \vskip -0.1in
\end{table*}

\begin{figure*}[t]
    \centering
    \includegraphics[width=0.98\linewidth]{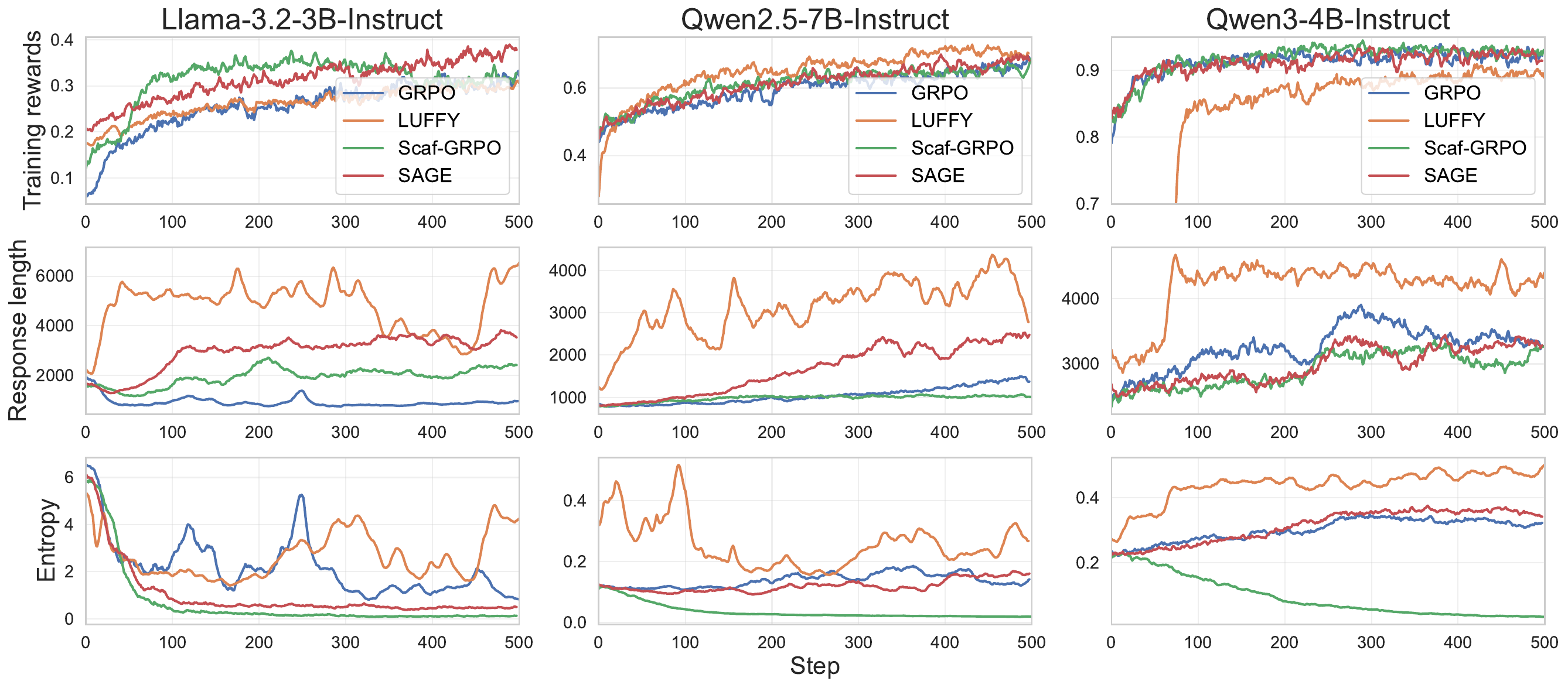}
    \vspace{-2mm}
    \caption{The training dynamics of different methods. For the training rewards, one should focus on the trend instead of the value, since adding hint (\textsc{SAGE}  and Scaf-GRPO) modifies the prompt difficulty, and using a correct off-policy trajectory (LUFFY) increases the reward. (1) LUFFY shows the most instability, with a very high entropy for Llama and a very low reward at the beginning of training for Qwen3, because it imitates the off-policy trajectory whose distribution might not be aligned with the policy model. (2) Scaf-GRPO shows the lowest entropy, implying less exploration. (3) \textsc{SAGE}  retains the on-policy characteristic, has a mild entropy and shows a stable growth in response length, which normally implies a better reasoning pattern.}
    \label{fig:trainnig_dynamics}
\end{figure*}

\section{Empirical Results}
\label{sec:empirical_results}

\textbf{Models.} We use LLMs with varying degrees of math specialization: Llama-3.2-3B-Instruct \citep{meta_llama3}, Qwen2.5-7B-Instruct \citep{qwen2}, and Qwen3-4B-Instruct-2507 \citep{yang2025qwen3}, representing low, moderate, and high levels of math-focused optimization, respectively, with the latter trained extensively via RL. 

\textbf{Training set.} Our training data are drawn from OpenR1-Math-220k \citep{openr1}, using prompts from NuminaMath 1.5 \citep{li2024numinamath} and reasoning traces generated by DeepSeek-R1 \citep{deepseekai2025deepseekr1incentivizingreasoningcapability}. The initial dataset contains 94k prompts. To ensure answer verifiability, we apply the \emph{Math-Verify} tool \citep{Kydlicek_Math-Verify_Math_Verification} to remove prompts whose DeepSeek-R1 reasoning traces are incorrectly verified, resulting in 64k prompts. These prompts are used in Figure \ref{fig:pass_k_0_during_training}. Due to limited  resources, we further subsample 15k prompts from this set, restricting the corresponding DeepSeek-R1 reasoning traces to fewer than 8,192 tokens. This constraint is necessary because one of our baselines, LUFFY \citep{yan2025learning}, relies on these reasoning traces, and excessively long traces would significantly increase RL training time. As we do not filter prompts based on pass rate, the resulting 15k prompts span a wide range of difficulty levels, resembling a practical RL training dataset.

\textbf{Evaluation sets.} We primarily evaluate our models on six widely used mathematical benchmarks: AIME24 \citep{aime24}, AIME25 \citep{aime25}, AMC23 \citep{li2024numinamath}, MATH-500 \citep{hendrycks2021measuring}, Minerva Math \citep{lewkowycz2022solving}, and OlympiadBench \citep{he2024olympiadbench}. In addition, we include two non-mathematical benchmarks, GPQA-diamond \citep{rein2024gpqa} and MMLU-Pro \citep{wang2024mmlu}, to assess the generalization ability of the trained models. \textbf{Notably, we only use hint during training. The prompt alone is the input to an LLM for evaluation.}

\textbf{Baselines.} We compare \textsc{SAGE}  with the following baselines: (1) Supervised Fine-Tuning (SFT), which finetunes the model on reasoning traces from DeepSeek-R1; (2) GRPO \citep{shao2024deepseekmath}, which learns without any hints; (3) LUFFY \citep{yan2025learning}, which replaces one on-policy trajectory with the corresponding correct trajectory from DeepSeek-R1; and (4) Scaf-GRPO \citep{zhang2025scaf}, which incorporates hints generated by GPT-5.2 under a low-reasoning-effort setting. Notably, SFT, LUFFY and Scaf-GRPO all rely on a stronger external LLM, whereas \textsc{SAGE}  learns only from self-generated hints. For fair comparison, we reproduce LUFFY and Scaf-GRPO using their open-source implementations on the same 15k sampled prompts, aligning only the batch size and number of training steps.

\textbf{Implementation details.} We run all experiments on 8 A100 GPUs, and use verl \citep{sheng2025hybridflow} for training and vLLM \citep{kwon2023efficient} for sampling. Following DAPO \citep{yu2025dapo}, we disable the KL term by setting $\beta = 0$, and apply asymmetric clipping with $\epsilon_{\text{low}} = 0.2$ and $\epsilon_{\text{high}} = 0.28$. Unless otherwise specified, the maximum response length is set to 8096 for both training and evaluation,\footnote{We use 8096 for the main results, as required by LUFFY, and 2048 for the remaining experiments due to resource constraints.} with a batch size of 128, 8 trajectories per prompt,\footnote{We use 4 trajectories for Qwen3-4B-Instruct due to slower training caused by its long response length.} and 500 training steps in total. We evaluate every 50 steps, and report the best average accuracy over all checkpoints. We set $L$ as 3, and $\alpha=0.35$ for \textsc{SAGE-light}. Details of the prompt used for hint generation and injection are provided in Appendix \ref{app:prompt_for_hint}. Complete training and evaluation settings for all methods are reported in Appendix \ref{app:detailed_implementation_settings}.

\subsection{Main results}
We report results for all methods and LLMs on eight benchmarks in Table \ref{tab:main_results}, with corresponding training dynamics shown in Figure \ref{fig:trainnig_dynamics}. Across three base models, \textsc{SAGE}  consistently achieves the highest average performance among all baselines, yielding improvements of +6.1 (Llama-3.2), +4.5 (Qwen2.5), and +4.2 (Qwen3) on average across the benchmarks. We use a fixed training set for all LLMs, despite their differing degrees of optimization for mathematical tasks. Consequently, the training set is relatively easier for Qwen3 and more challenging for Llama, a discrepancy that is also reflected in Figure \ref{fig:pass_k_0_during_training}. Nevertheless, \textsc{SAGE}  consistently improves performance across all LLMs, demonstrating robust and effective generalization.

\textbf{SAGE vs. SFT.} SFT yields the worst performance, underperforming even the base LLM, due to its tendency to overfit training data. In contrast, \textsc{SAGE}  preserves the RL characteristics, selectively sharpening the model’s distribution to correct trajectories.

\begin{table}[t]
    \caption{Percentage of prompts without any training signal, i.e., not any correct trajectories of these prompts are sampled during the whole training procedure.}
    \label{tab:useless_prompts}
    \vspace{-2mm}
    \scriptsize
    \begin{center}
    \begin{tabular}{lrrr}
    \toprule
    \textbf{Method} & \textbf{Llama-3.2-3B} & \textbf{Qwen2.5-7B} & \textbf{Qwen3-4B} \\
    \midrule
    Base & 56.9\% & 29.8\% & 21.8\% \\
    GRPO &  40.2\% & 10.3\% & 1.3\% \\
    \textsc{SAGE}  & 30.0\% & 8.2\% & 1.0\% \\
    \bottomrule
    \end{tabular}
    \end{center}
    \vskip -0.1in
\end{table}

\textbf{SAGE vs. GRPO.} Table~\ref{tab:useless_prompts} reports the proportion of prompts that never provide a training signal. Compared to GRPO, \textsc{SAGE} makes substantially more effective use of the prompt set. This effect is particularly pronounced for the weaker LLM, Llama-3.2: by leveraging self-generated hints, \textsc{SAGE}  successfully utilizes 10\% more prompts, leading to the largest performance improvement over GRPO (+2.0). For the stronger model, Qwen3, \textsc{SAGE}  behaves more similarly to GRPO, with nearly identical prompt utilization. Nevertheless, despite using only 0.3\% more prompts, \textsc{SAGE}  still achieves a +1.3 accuracy gain over GRPO. These hard prompts play a critical role in RL, which aligns with prior work \citep{xiong2025reinforce, yu2025dapo} that favors RL on prompt sets with lower pass rates. Furthermore, in Figure \ref{fig:trainnig_dynamics}, \textsc{SAGE}  exhibits faster response-length growth than GRPO for both Llama-3.2 and Qwen2.5,  due to learning from hard prompts that fail to provide any signal under GRPO.

\textbf{SAGE vs. LUFFY.} LUFFY exhibits the second-largest degree of off-policy behavior, following SFT, as one of its trajectories is generated by a different model. In Figure \ref{fig:trainnig_dynamics}, the response length increases dramatically at the early stage of training, reflecting the LLM’s tendency to imitate the stronger model. However, this off-policy setting introduces training instability due to the misalignment between the policy model and the stronger model. Specifically, Llama-3.2 trained with LUFFY displays excessively high entropy and highly oscillatory response lengths, while Qwen3 trained with LUFFY suffers from very low rewards at the beginning of training. As reported in Table \ref{tab:main_results}, LUFFY only outperforms GRPO (while still underperforming \textsc{SAGE}) on Qwen2.5, and performs worse than the base model on both Llama-3.2 and Qwen3.

\textbf{SAGE vs. Scaf-GRPO.} Scaf-GRPO relies on hints generated by a stronger model (e.g., GPT-5.2 in our setting). As shown in Figure \ref{fig:trainnig_dynamics}, it exhibits the lowest entropy among all methods, indicating limited exploration. This behavior may stem from the hints revealing excessive information. In contrast, \textsc{SAGE}  maintains an entropy level comparable to GRPO,\footnote{The entropy of GRPO on Llama-3.2 is abnormally high.} while consistently outperforming Scaf-GRPO in Table \ref{tab:main_results}. Moreover, learning from self-generated hints enables a more end-to-end training procedure and simplifies implementation.

\textbf{SAGE vs, \textsc{SAGE-light}.} \textsc{SAGE-light} achieves slightly lower accuracy but improves efficiency, requiring 53\% of SAGE’s training time.

\textbf{Out-of-distribution performance.} In Table \ref{tab:main_results}, performance on out-of-distribution benchmarks shows a similar pattern to that on in-distribution ones.  SAGE and \textsc{SAGE-light} consistently achieve the best and second-best accuracy, respectively, indicating superior generalization capability.

\begin{table}[t]
    \caption{Training time on Qwen2.5-7B-Instruct.}
    \label{tab:latency}
    \vspace{-2mm}
    \scriptsize
    \begin{center}
    \begin{tabular}{ccccc}
    \toprule
    GRPO & LUFFY & Scaf-GRPO & SAGE & \textsc{SAGE-light} \\
    1.0$\times$ (25.3h) & 1.2$\times$ & 1.5$\times$ & 2.3$\times$ & 1.2$\times$ \\
    \bottomrule
    \end{tabular}
    \end{center}
    \vskip -0.1in
\end{table}

\subsection{Discussion}\label{sec:dis}
\textbf{Latency.} A potential limitation of SAGE is its latency, as it must generate and use hints on the fly when a correct trajectory of the prompt can't be sampled. Table~\ref{tab:latency} reports the training time of different RL methods. Among them, SAGE incurs the highest training cost, while \textsc{SAGE-light} requires only slightly more time than GRPO. For highly complex prompts, SAGE may sample hints across multiple levels (from $l=0$ to $l=3$), which increases computational overhead. In contrast, \textsc{SAGE-light} leverages the prompt accuracy from the previous epoch to select an appropriate hint level, and thus samples from only a single level. These two SAGE variants provide flexible trade-offs for practitioners with different efficiency requirements, and both consistently outperform the baseline methods.

\begin{figure}[t]
    \centering
    \includegraphics[width=0.98\linewidth]{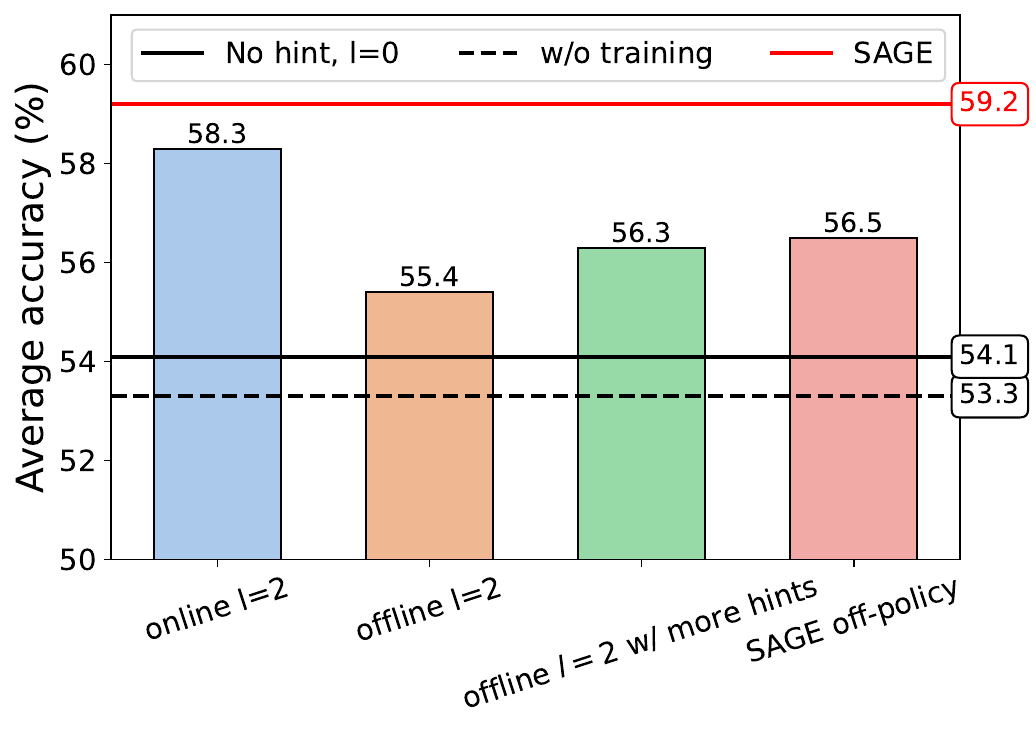}
    \caption{Ablation studies on Qwen3-4B-Instruct trained with the same prompt set as Figure \ref{fig:poc_hints}.}
    \label{fig:ablation}
    \vspace{-2mm}
\end{figure}

\begin{figure}[t]
    \centering
    \includegraphics[width=0.98\linewidth]{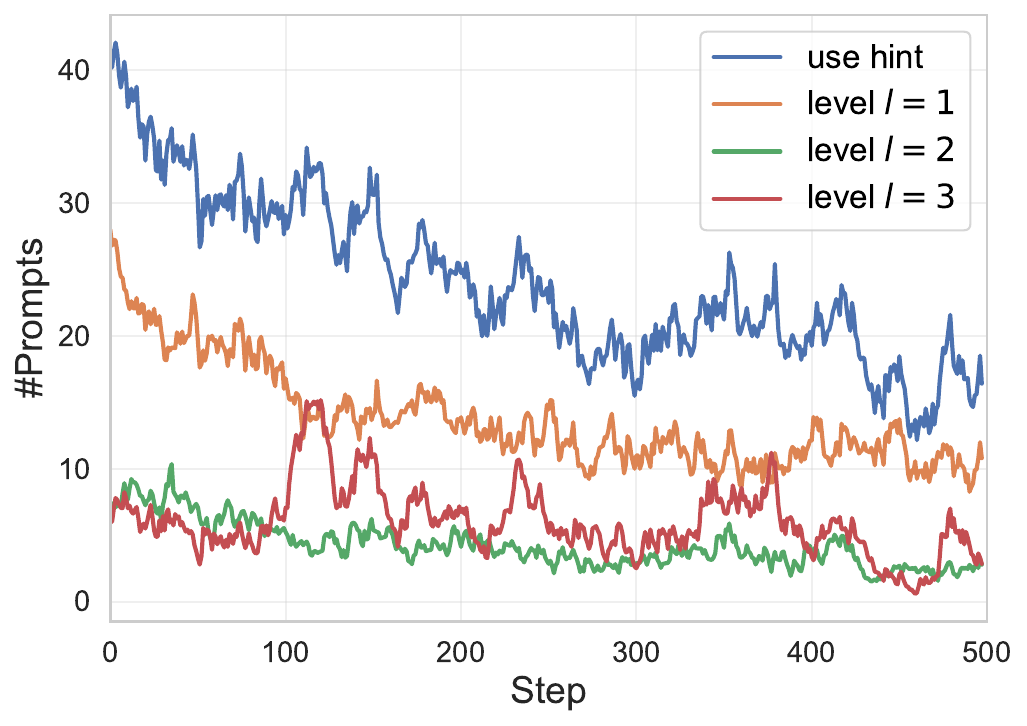}
    \caption{Number of prompts uses hint during training on Llama-3.2-3B-Instruct. Batch size is 128. The model use less hint w.r.t. the training step, indicating that the model becomes more powerful.}
    \label{fig:training_dynamics_level}
    \vspace{-2mm}
\end{figure}

\textbf{Offline with more hints.} In Figure~\ref{fig:poc_hints}, online self-hinting generates a new hint at each training step, whereas offline self-hinting relies on a fixed hint generated prior to training. One possible explanation for the superior performance of online self-hinting is the increased diversity of hints. To examine this, we have an additional ablation (Figure~\ref{fig:ablation}): offline self-hinting with multiple hints. Specifically, before training, we use the base LLM to generate 10 hints with \texttt{temperature=1.0}. During training, a different hint is used for the same prompt at each step, ensuring that identical prompts are paired with diverse hints over time. We observe that increased hint diversity indeed improves performance, yielding a +0.9 gain over standard offline self-hinting. Nevertheless, online self-hinting still outperforms this variant by a margin of 2.0. We argue that hint generated in an online manner offer more benefit than diversity.

\textbf{Off-policy.} We ablate whether the policy-gradient matches the sampling context: \textsc{SAGE} optimizes $\log \pi_\theta(\cdot\mid x,h)$, while an off-policy variant rolls out with $h$ but optimizes $\log \pi_\theta(\cdot\mid x)$. Figure~\ref{fig:ablation} shows a clear drop for the off-policy variant (56.5) compared with on-policy \textsc{SAGE} (59.2) and even the single level hint level baseline (58.3).

\textbf{Same level hint vs. SAGE.} SAGE does not rely on a fixed hint level. Instead, it adaptively increases the hint level only when the current level fails to yield a correct response. In Figure \ref{fig:ablation}, this strategy leads to a clear performance gain over using a constant hint level (e.g., online $l=2$), achieving an improvement of +0.9. This design enables more effective utilization of the hard prompt, allowing the LLM to progressively learn from w\textbf{}eaker hints, down to $l=0$.

\textbf{Less hint w.r.t. step.} In Figure \ref{fig:training_dynamics_level}, we can observe that the LLM use less and less hint during the training. It indicates that self-hinting indeed enhances RL. LLM becomes more and more powerful and can gradually solve difficult problems without the help of a hint. 

\textbf{Case study.} An example on how hint helps (Appendix \ref{app:case_study}).

\section{Related Work}
\textbf{Data resampling and external guidance.} 
Data selection and filtering are widely used in online RL for LLMs \citep{zhang2024policy,dong2023raft,xiong2023iterative,dong2024rlhf,shi2024crucial,liao2025reward,feng2025pilaf}, and become particularly important for GRPO-style methods where groupwise advantages can collapse under sparse rewards. Prior work mitigates this issue mainly by reshaping the training distribution or injecting external guidance. A common workaround is to skip degenerate groups and resample or upweight prompts \citep{yu2025dapo,xiong2025minimalist,yao2025optimizing,li2025knapsack}, which improves efficiency but biases training toward prompts with non-trivial success probability. Another direction bootstraps learning by adding positive trajectories from stronger teachers, reference models, or offline buffers \citep{zhang2025stephint,yan2025learning}, but this can introduce context or distribution mismatch when mixed with on-policy rollouts. In contrast, \textsc{SAGE}  preserves a clean on-policy objective by changing the rollout distribution through privileged hinting, without discarding hard prompts or relying on static external buffers.

\textbf{Privileged hinting and \textsc{SAGE}.} While leveraging intermediate guidance, such as plans or gold solutions, has a rich history in RL \citep{ng1999policy,szepesvari2022algorithms,ouyang2022training}, recent LLM-specific adaptations often implement hinting via heuristic ``batch surgery." For instance, \citet{zhang2025scaf} mitigates collapse by augmenting rollout batches with hinted trajectories upon detecting failure. This approach mixes contexts (e.g., $x$ and $x,h$ ) within a single group, which blurs the interpretation of groupwise baselines and advantage normalization. Furthermore, external hint generators may not be calibrated to the learner's current capabilities. \textsc{SAGE}  distinguishes itself through three key design choices: (1) it formalizes hinting as an explicitly augmented on-policy sampling process, ensuring the GRPO loss remains well-defined; (2) it employs a policy-dependent strength scheduler that activates hints only when the "learning gate" is closed; and (3) it utilizes online self-hinting via a lagged policy, ensuring the hint distribution tracks the learner’s evolving support rather than relying on a potentially misaligned external teacher.

\section{Conclusion}

We identify a finite-sample degeneracy in GRPO under sparse 0/1 rewards. When group rewards are identical, advantage standardization collapse, and the minibatch gradient vanishes on hard prompts. We propose \textsc{SAGE}, a privileged procedural hinting method that injects reference-solution-derived hints during training to shift the rollout distribution while preserving the original reward definition. A policy-dependent schedule gates hint strength based on detected group collapse, and inference uses the no-hint policy. Extensive experiments validate the improvements across tasks.

\section*{Acknowledgements}
This research was partly supported by the Netherlands Organization for Scientific Research (NWO) under project
number VI.C.192.080.

\section*{Impact Statements}
This work can reduce training cost and improve stability of RL for LLMs on verifiable tasks. Risks include misuse to optimize harmful verifiable objectives.

\bibliography{myrefs}

@article{hu2024openrlhf,
  title={OpenRLHF: An Easy-to-use, Scalable and High-performance RLHF Framework},
  author={Jian Hu and Xibin Wu and Zilin Zhu and Xianyu and Weixun Wang and Dehao Zhang and Yu Cao},
  journal={arXiv preprint arXiv:2405.11143},
  year={2024}
}

@inproceedings{sheng2025hybridflow,
  title={Hybridflow: A flexible and efficient rlhf framework},
  author={Sheng, Guangming and Zhang, Chi and Ye, Zilingfeng and Wu, Xibin and Zhang, Wang and Zhang, Ru and Peng, Yanghua and Lin, Haibin and Wu, Chuan},
  booktitle={Proceedings of the Twentieth European Conference on Computer Systems},
  pages={1279--1297},
  year={2025}
}

@article{wang2024mmlu,
  title={Mmlu-pro: A more robust and challenging multi-task language understanding benchmark},
  author={Wang, Yubo and Ma, Xueguang and Zhang, Ge and Ni, Yuansheng and Chandra, Abhranil and Guo, Shiguang and Ren, Weiming and Arulraj, Aaran and He, Xuan and Jiang, Ziyan and others},
  journal={Advances in Neural Information Processing Systems},
  volume={37},
  pages={95266--95290},
  year={2024}
}

@inproceedings{rein2024gpqa,
  title={Gpqa: A graduate-level google-proof q\&a benchmark},
  author={Rein, David and Hou, Betty Li and Stickland, Asa Cooper and Petty, Jackson and Pang, Richard Yuanzhe and Dirani, Julien and Michael, Julian and Bowman, Samuel R},
  booktitle={First Conference on Language Modeling},
  year={2024}
}

@misc{aime25,
  author       = {{MAA Committees}},
  title        = {{AIME Problems and Solutions}},
  howpublished = {\url{https://artofproblemsolving.com/wiki/index.php/AIME_Problems_and_Solutions}},
  year={2025},
}

@misc{aime24,
  author       = {{MAA Committees}},
  title        = {{AIME Problems and Solutions}},
  howpublished = {\url{https://artofproblemsolving.com/wiki/index.php/AIME_Problems_and_Solutions}},
  year={2024},
}

@article{li2024numinamath,
  title={Numinamath: The largest public dataset in ai4maths with 860k pairs of competition math problems and solutions},
  author={Li, Jia and Beeching, Edward and Tunstall, Lewis and Lipkin, Ben and Soletskyi, Roman and Huang, Shengyi and Rasul, Kashif and Yu, Longhui and Jiang, Albert Q and Shen, Ziju and others},
  journal={Hugging Face repository},
  volume={13},
  number={9},
  pages={9},
  year={2024}
}

@article{yang2025qwen3,
  title={Qwen3 technical report},
  author={Yang, An and Li, Anfeng and Yang, Baosong and Zhang, Beichen and Hui, Binyuan and Zheng, Bo and Yu, Bowen and Gao, Chang and Huang, Chengen and Lv, Chenxu and others},
  journal={arXiv preprint arXiv:2505.09388},
  year={2025}
}

@misc{openr1,
    title = {Open R1: A fully open reproduction of DeepSeek-R1},
    url = {https://github.com/huggingface/open-r1},
    author = {{Hugging Face}},
    month = {January},
    year = {2025}
}

@article{liao2025reward,
  title={Reward-guided speculative decoding for efficient llm reasoning},
  author={Liao, Baohao and Xu, Yuhui and Dong, Hanze and Li, Junnan and Monz, Christof and Savarese, Silvio and Sahoo, Doyen and Xiong, Caiming},
  journal={arXiv preprint arXiv:2501.19324},
  year={2025}
}

@software{Kydlicek_Math-Verify_Math_Verification,
author = {Kydlíček, Hynek},
license = {Apache-2.0},
title = {{Math-Verify: Math Verification Library}},
url = {https://github.com/huggingface/math-verify},
version = {0.6.1}
}

@article{yao2025optimizing,
  title={Optimizing Chain-of-Thought Reasoners via Gradient Variance Minimization in Rejection Sampling and RL},
  author={Yao, Jiarui and Hao, Yifan and Zhang, Hanning and Dong, Hanze and Xiong, Wei and Jiang, Nan and Zhang, Tong},
  journal={arXiv preprint arXiv:2505.02391},
  year={2025}
}

@article{xiong2025minimalist,
  title={A minimalist approach to llm reasoning: from rejection sampling to reinforce},
  author={Xiong, Wei and Yao, Jiarui and Xu, Yuhui and Pang, Bo and Wang, Lei and Sahoo, Doyen and Li, Junnan and Jiang, Nan and Zhang, Tong and Xiong, Caiming and others},
  journal={arXiv preprint arXiv:2504.11343},
  year={2025}
}

@misc{deepseekai2025deepseekr1incentivizingreasoningcapability,
      title={DeepSeek-R1: Incentivizing Reasoning Capability in LLMs via Reinforcement Learning}, 
      author={DeepSeek-AI and Daya Guo and Dejian Yang and Haowei Zhang and Junxiao Song and Ruoyu Zhang and Runxin Xu and Qihao Zhu and Shirong Ma and Peiyi Wang and Xiao Bi and Xiaokang Zhang and Xingkai Yu and Yu Wu and Z. F. Wu and Zhibin Gou and Zhihong Shao and Zhuoshu Li and Ziyi Gao and Aixin Liu and Bing Xue and Bingxuan Wang and Bochao Wu and Bei Feng and Chengda Lu and Chenggang Zhao and Chengqi Deng and Chenyu Zhang and Chong Ruan and Damai Dai and Deli Chen and Dongjie Ji and Erhang Li and Fangyun Lin and Fucong Dai and Fuli Luo and Guangbo Hao and Guanting Chen and Guowei Li and H. Zhang and Han Bao and Hanwei Xu and Haocheng Wang and Honghui Ding and Huajian Xin and Huazuo Gao and Hui Qu and Hui Li and Jianzhong Guo and Jiashi Li and Jiawei Wang and Jingchang Chen and Jingyang Yuan and Junjie Qiu and Junlong Li and J. L. Cai and Jiaqi Ni and Jian Liang and Jin Chen and Kai Dong and Kai Hu and Kaige Gao and Kang Guan and Kexin Huang and Kuai Yu and Lean Wang and Lecong Zhang and Liang Zhao and Litong Wang and Liyue Zhang and Lei Xu and Leyi Xia and Mingchuan Zhang and Minghua Zhang and Minghui Tang and Meng Li and Miaojun Wang and Mingming Li and Ning Tian and Panpan Huang and Peng Zhang and Qiancheng Wang and Qinyu Chen and Qiushi Du and Ruiqi Ge and Ruisong Zhang and Ruizhe Pan and Runji Wang and R. J. Chen and R. L. Jin and Ruyi Chen and Shanghao Lu and Shangyan Zhou and Shanhuang Chen and Shengfeng Ye and Shiyu Wang and Shuiping Yu and Shunfeng Zhou and Shuting Pan and S. S. Li and Shuang Zhou and Shaoqing Wu and Shengfeng Ye and Tao Yun and Tian Pei and Tianyu Sun and T. Wang and Wangding Zeng and Wanjia Zhao and Wen Liu and Wenfeng Liang and Wenjun Gao and Wenqin Yu and Wentao Zhang and W. L. Xiao and Wei An and Xiaodong Liu and Xiaohan Wang and Xiaokang Chen and Xiaotao Nie and Xin Cheng and Xin Liu and Xin Xie and Xingchao Liu and Xinyu Yang and Xinyuan Li and Xuecheng Su and Xuheng Lin and X. Q. Li and Xiangyue Jin and Xiaojin Shen and Xiaosha Chen and Xiaowen Sun and Xiaoxiang Wang and Xinnan Song and Xinyi Zhou and Xianzu Wang and Xinxia Shan and Y. K. Li and Y. Q. Wang and Y. X. Wei and Yang Zhang and Yanhong Xu and Yao Li and Yao Zhao and Yaofeng Sun and Yaohui Wang and Yi Yu and Yichao Zhang and Yifan Shi and Yiliang Xiong and Ying He and Yishi Piao and Yisong Wang and Yixuan Tan and Yiyang Ma and Yiyuan Liu and Yongqiang Guo and Yuan Ou and Yuduan Wang and Yue Gong and Yuheng Zou and Yujia He and Yunfan Xiong and Yuxiang Luo and Yuxiang You and Yuxuan Liu and Yuyang Zhou and Y. X. Zhu and Yanhong Xu and Yanping Huang and Yaohui Li and Yi Zheng and Yuchen Zhu and Yunxian Ma and Ying Tang and Yukun Zha and Yuting Yan and Z. Z. Ren and Zehui Ren and Zhangli Sha and Zhe Fu and Zhean Xu and Zhenda Xie and Zhengyan Zhang and Zhewen Hao and Zhicheng Ma and Zhigang Yan and Zhiyu Wu and Zihui Gu and Zijia Zhu and Zijun Liu and Zilin Li and Ziwei Xie and Ziyang Song and Zizheng Pan and Zhen Huang and Zhipeng Xu and Zhongyu Zhang and Zhen Zhang},
      year={2025},
      eprint={2501.12948},
      archivePrefix={arXiv},
      primaryClass={cs.CL},
      url={https://arxiv.org/abs/2501.12948}, 
}

@article{meta_llama3,
  title={Introducing Meta Llama 3: The most capable openly available LLM to date},
  author={Meta},
  journal={Meta AI Blog},
  year={2024},
  note={\url{https://ai.meta.com/blog/meta-llama-3/}}
}

@article{shi2024crucial,
  title={The crucial role of samplers in online direct preference optimization},
  author={Shi, Ruizhe and Zhou, Runlong and Du, Simon S},
  journal={arXiv preprint arXiv:2409.19605},
  year={2024}
}

@article{feng2025pilaf,
  title={Pilaf: Optimal human preference sampling for reward modeling},
  author={Feng, Yunzhen and Kwiatkowski, Ariel and Zheng, Kunhao and Kempe, Julia and Duan, Yaqi},
  journal={arXiv preprint arXiv:2502.04270},
  year={2025}
}

@article{zhang2024policy,
  title={Policy filtration in rlhf to fine-tune llm for code generation},
  author={Zhang, Chuheng and Shen, Wei and Zhao, Li and Zhang, Xuyun and Qi, Lianyong and Dou, Wanchun and Bian, Jiang}
}

@article{liu2025understanding,
  title={Understanding r1-zero-like training: A critical perspective},
  author={Liu, Zichen and Chen, Changyu and Li, Wenjun and Qi, Penghui and Pang, Tianyu and Du, Chao and Lee, Wee Sun and Lin, Min},
  journal={arXiv preprint arXiv:2503.20783},
  year={2025}
}

@article{he2024olympiadbench,
  title={Olympiadbench: A challenging benchmark for promoting agi with olympiad-level bilingual multimodal scientific problems},
  author={He, Chaoqun and Luo, Renjie and Bai, Yuzhuo and Hu, Shengding and Thai, Zhen Leng and Shen, Junhao and Hu, Jinyi and Han, Xu and Huang, Yujie and Zhang, Yuxiang and others},
  journal={arXiv preprint arXiv:2402.14008},
  year={2024}
}

@article{lewkowycz2022solving,
  title={Solving quantitative reasoning problems with language models},
  author={Lewkowycz, Aitor and Andreassen, Anders and Dohan, David and Dyer, Ethan and Michalewski, Henryk and Ramasesh, Vinay and Slone, Ambrose and Anil, Cem and Schlag, Imanol and Gutman-Solo, Theo and others},
  journal={Advances in Neural Information Processing Systems},
  volume={35},
  pages={3843--3857},
  year={2022}
}

@article{qwen2,
  title={Qwen2. 5 technical report},
  author={Yang, An and Yang, Baosong and Zhang, Beichen and Hui, Binyuan and Zheng, Bo and Yu, Bowen and Li, Chengyuan and Liu, Dayiheng and Huang, Fei and Wei, Haoran and others},
  journal={arXiv preprint arXiv:2412.15115},
  year={2024}
}

@article{shao2024deepseekmath,
  title={Deepseekmath: Pushing the limits of mathematical reasoning in open language models},
  author={Shao, Zhihong and Wang, Peiyi and Zhu, Qihao and Xu, Runxin and Song, Junxiao and Zhang, Mingchuan and Li, YK and Wu, Y and Guo, Daya},
  journal={arXiv preprint arXiv:2402.03300},
  year={2024}
}

@inproceedings{kwon2023efficient,
  title={Efficient Memory Management for Large Language Model Serving with PagedAttention},
  author={Woosuk Kwon and Zhuohan Li and Siyuan Zhuang and Ying Sheng and Lianmin Zheng and Cody Hao Yu and Joseph E. Gonzalez and Hao Zhang and Ion Stoica},
  booktitle={Proceedings of the ACM SIGOPS 29th Symposium on Operating Systems Principles},
  year={2023}
}

@article{schulman2017proximal,
  title={Proximal policy optimization algorithms},
  author={Schulman, John and Wolski, Filip and Dhariwal, Prafulla and Radford, Alec and Klimov, Oleg},
  journal={arXiv preprint arXiv:1707.06347},
  year={2017}
}

@article{
dong2023raft,
title={{RAFT}: Reward rAnked FineTuning for Generative Foundation Model Alignment},
author={Hanze Dong and Wei Xiong and Deepanshu Goyal and Yihan Zhang and Winnie Chow and Rui Pan and Shizhe Diao and Jipeng Zhang and KaShun SHUM and Tong Zhang},
journal={Transactions on Machine Learning Research},
issn={2835-8856},
year={2023},
url={https://openreview.net/forum?id=m7p5O7zblY},
note={}
}

@article{ouyang2022training,
  title={Training language models to follow instructions with human feedback},
  author={Ouyang, Long and Wu, Jeffrey and Jiang, Xu and Almeida, Diogo and Wainwright, Carroll and Mishkin, Pamela and Zhang, Chong and Agarwal, Sandhini and Slama, Katarina and Ray, Alex and others},
  journal={Advances in Neural Information Processing Systems},
  volume={35},
  pages={27730--27744},
  year={2022}
}

@inproceedings{xiong2023iterative,
  title={Iterative Preference Learning from Human Feedback: Bridging Theory and Practice for RLHF under KL-constraint},
  author={Xiong, Wei and Dong, Hanze and Ye, Chenlu and Wang, Ziqi and Zhong, Han and Ji, Heng and Jiang, Nan and Zhang, Tong},
  joural={ICLR 2024 Workshop on Mathematical and Empirical Understanding of Foundation Models},
    year={2023}
}

@article{hendrycks2021measuring,
  title={Measuring mathematical problem solving with the math dataset},
  author={Hendrycks, Dan and Burns, Collin and Kadavath, Saurav and Arora, Akul and Basart, Steven and Tang, Eric and Song, Dawn and Steinhardt, Jacob},
  journal={arXiv preprint arXiv:2103.03874},
  year={2021}
}

@article{dong2024rlhf,
  title={Rlhf workflow: From reward modeling to online rlhf},
  author={Dong, Hanze and Xiong, Wei and Pang, Bo and Wang, Haoxiang and Zhao, Han and Zhou, Yingbo and Jiang, Nan and Sahoo, Doyen and Xiong, Caiming and Zhang, Tong},
  journal={arXiv preprint arXiv:2405.07863},
  year={2024}
}

@article{yu2025dapo,
  title={Dapo: An open-source llm reinforcement learning system at scale},
  author={Yu, Qiying and Zhang, Zheng and Zhu, Ruofei and Yuan, Yufeng and Zuo, Xiaochen and Yue, Yu and Fan, Tiantian and Liu, Gaohong and Liu, Lingjun and Liu, Xin and others},
  journal={arXiv preprint arXiv:2503.14476},
  year={2025}
}

@inproceedings{ng1999policy,
  title={Policy invariance under reward transformations: Theory and application to reward shaping},
  author={Ng, Andrew Y and Harada, Daishi and Russell, Stuart},
  booktitle={Icml},
  volume={99},
  pages={278--287},
  year={1999},
  organization={Citeseer}
}

@book{szepesvari2022algorithms,
  title={Algorithms for reinforcement learning},
  author={Szepesv{\'a}ri, Csaba},
  year={2022},
  publisher={Springer nature}
}

@article{zhang2025stephint,
  title={StepHint: Multi-level Stepwise Hints Enhance Reinforcement Learning to Reason},
  author={Zhang, Kaiyi and Lv, Ang and Li, Jinpeng and Wang, Yongbo and Wang, Feng and Hu, Haoyuan and Yan, Rui},
  journal={arXiv preprint arXiv:2507.02841},
  year={2025}
}

@article{yan2025learning,
  title={Learning to reason under off-policy guidance},
  author={Yan, Jianhao and Li, Yafu and Hu, Zican and Wang, Zhi and Cui, Ganqu and Qu, Xiaoye and Cheng, Yu and Zhang, Yue},
  journal={arXiv preprint arXiv:2504.14945},
  year={2025}
}

@article{zhang2025scaf,
  title={Scaf-GRPO: Scaffolded Group Relative Policy Optimization for Enhancing LLM Reasoning},
  author={Zhang, Xichen and Wu, Sitong and Zhu, Yinghao and Tan, Haoru and Yu, Shaozuo and He, Ziyi and Jia, Jiaya},
  journal={arXiv preprint arXiv:2510.19807},
  year={2025}
}

@article{zhang2025improving,
  title={Improving sampling efficiency in rlvr through adaptive rollout and response reuse},
  author={Zhang, Yuheng and Yao, Wenlin and Yu, Changlong and Liu, Yao and Yin, Qingyu and Yin, Bing and Yun, Hyokun and Li, Lihong},
  journal={arXiv preprint arXiv:2509.25808},
  year={2025}
}

@article{li2025knapsack,
  title={Knapsack rl: Unlocking exploration of llms via optimizing budget allocation},
  author={Li, Ziniu and Chen, Congliang and Yang, Tianyun and Ding, Tian and Sun, Ruoyu and Zhang, Ge and Huang, Wenhao and Luo, Zhi-Quan},
  journal={arXiv preprint arXiv:2509.25849},
  year={2025}
}

@article{xiong2025reinforce,
  title={Reinforce-Ada: An Adaptive Sampling Framework under Non-linear RL Objectives},
  author={Xiong, Wei and Ye, Chenlu and Liao, Baohao and Dong, Hanze and Xu, Xinxing and Monz, Christof and Bian, Jiang and Jiang, Nan and Zhang, Tong},
  journal={arXiv preprint arXiv:2510.04996},
  year={2025}
}
\bibliographystyle{icml2026}

%%%%%%%%%%%%%%%%%%%%%%%%%%%%%%%%%%%%%%%%%%%%%%%%%%%%%%%%%%%%%%%%%%%%%%%%%%%%%%%
%%%%%%%%%%%%%%%%%%%%%%%%%%%%%%%%%%%%%%%%%%%%%%%%%%%%%%%%%%%%%%%%%%%%%%%%%%%%%%%
% APPENDIX
%%%%%%%%%%%%%%%%%%%%%%%%%%%%%%%%%%%%%%%%%%%%%%%%%%%%%%%%%%%%%%%%%%%%%%%%%%%%%%%
%%%%%%%%%%%%%%%%%%%%%%%%%%%%%%%%%%%%%%%%%%%%%%%%%%%%%%%%%%%%%%%%%%%%%%%%%%%%%%%
\newpage
\appendix

\counterwithin{figure}{section}
\counterwithin{table}{section}

\onecolumn

\section{Illustration of Privileged Hinting}
\label{app:case_study}

\begin{figure*}[th]
    \centering
        \includegraphics[width=\linewidth]{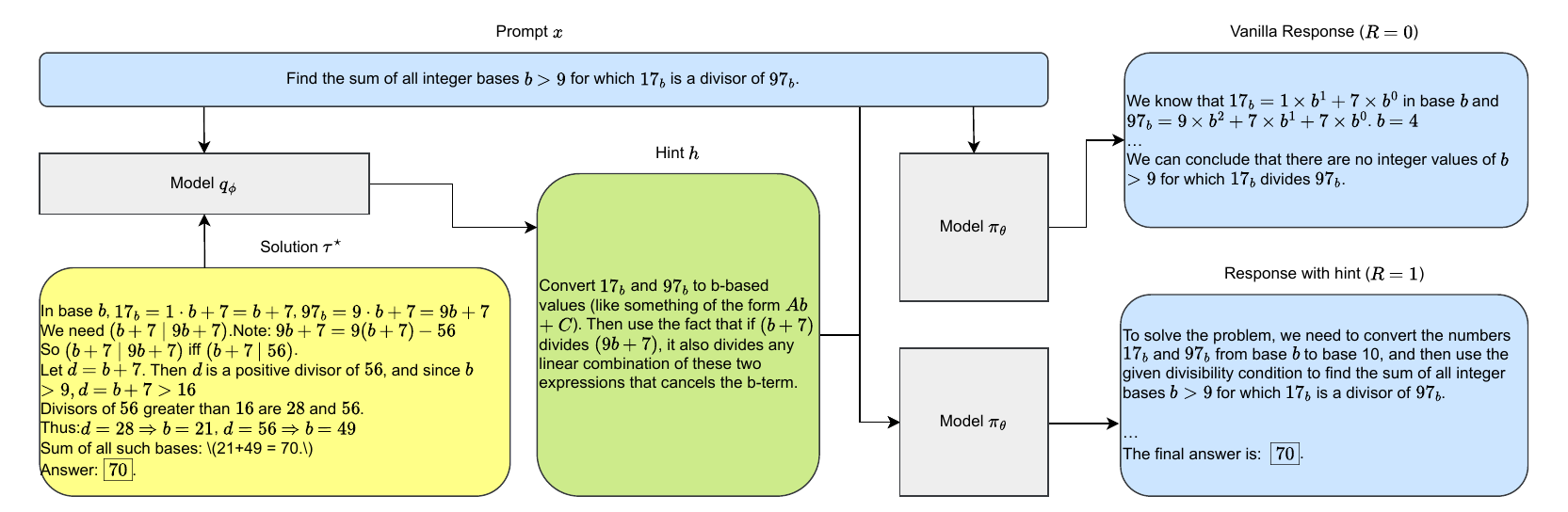}
  \caption{Example of privileged hinting for a single prompt.
Given a math prompt \(x\), a hint generator \(q_\phi\) uses the reference solution \(\tau^\star\) during training to produce a procedural hint (h) that summarizes intermediate reasoning without revealing the final answer. Rolling out the policy \(\pi_\theta\) on the original prompt can yield an incorrect solution with zero reward, while conditioning on \((x,h)\) shifts the rollout distribution and enables a correct solution with positive reward. The task return is unchanged, and at deployment the hint is removed so the model runs on the original prompt only.
}
  \label{fig:draw}
\end{figure*}

\begin{tcolorbox}[
  title=Case Study: Progressive privileged hints for $17_b \mid 97_b$,
  colback=white, colframe=black, boxrule=0.5mm
]
\small

\textbf{Prompt.} Find the sum of all integer bases $b>9$ for which $17_b$ is a divisor of $97_b$.

\medskip
\textbf{Reference solution (available only during training).}
In base $b$,
\[
17_b = 1\cdot b + 7 = b+7,\qquad
97_b = 9\cdot b + 7 = 9b+7.
\]
We need $b+7 \mid 9b+7$. Note that
\[
9b+7 = 9(b+7) - 56,
\]
hence $b+7 \mid 9b+7$ iff $b+7 \mid 56$.
Let $d=b+7$. Then $d$ is a positive divisor of $56$, and since $b>9$ we have $d=b+7>16$.
The divisors of $56$ greater than $16$ are $28$ and $56$, so $b\in\{21,49\}$ and the sum is $70$.

\medskip
\textbf{Privileged hints (used only during training; never shown at test time).}
\begin{itemize}
  \item \textbf{Level 1 (minimal).} Rewrite the base-$b$ numerals as ordinary integers in terms of $b$,
  then turn the divisibility condition into a statement about a simple linear expression.
  \item \textbf{Level 2 (medium).} Convert $17_b$ and $97_b$ into the form $Ab+C$.
  If $(b+7)$ divides $(9b+7)$, it also divides any linear combination of these expressions that cancels the $b$-term.
  \item \textbf{Level 3 (detailed).} Compute $17_b=b+7$ and $97_b=9b+7$.
  Subtract a multiple of $(b+7)$ from $(9b+7)$ to eliminate $b$, e.g.,
  $(9b+7)-9(b+7)$. The condition becomes ``$(b+7)$ divides a constant''. Enumerate divisors and keep only $b>9$.
\end{itemize}

\medskip
\textbf{Representative model behaviors (illustrative).}
\begin{itemize}
  \item \textbf{No hint ($h=\varnothing$):} the model may mis-expand $97_b$ (e.g., treating it as a multi-digit polynomial),
  and incorrectly conclude there is no valid $b>9$, yielding a wrong terminal decision and thus reward $0$. 
  \item \textbf{With stronger hints (Level 2/3):} the model is steered toward the key cancellation step
  $9b+7-9(b+7)=-56$, after which it can enumerate divisors and recover the correct bases and final sum ($70$). 
\end{itemize}

\end{tcolorbox}

\paragraph{Case Study: Privileged Hinting on a Simple Verifiable Math Task.}
To make the idea of \emph{privileged hinting} concrete, we include a small case study on a
verifiable divisibility question. The key point is that hints are \emph{training-time privileged context}:
they do not modify the verifier or the terminal reward. Instead, they reshape the rollout distribution
so that, under finite sampling, the policy is more likely to generate informative trajectories (e.g.,
those that perform the correct algebraic reduction). In practice, without hints the model may repeatedly
take an incorrect representation path (e.g., mis-expanding base-$b$ numerals) and receive identical
zero rewards across a rollout group, causing GRPO advantages to collapse. With progressive hints,
the model is guided toward the correct reduction, increasing the chance that a group contains mixed
outcomes and thus yields a non-degenerate update.

This example matches the operational role of \textsc{SAGE}: hints do not change the verifier reward,
but they increase the probability that at least one rollout in a finite group follows a \emph{useful}
trajectory (here, the cancellation-and-divisor-enumeration path). This increases within-group outcome
diversity, reduces the frequency of degenerate all-zero groups, and therefore prevents standardized
GRPO advantages from collapsing on hard prompts under finite sampling. 

\section{Prompt for hint generation and the usage of hint}
\label{app:prompt_for_hint}

\begin{figure}[H]
    \centering
    \small
    \begin{tcolorbox}[title=System prompt for hint generation, colback=white, colframe=black, boxrule=0.5mm]
    
    You are a tutoring assistant that generates progressive hints to help students solve difficult problems without revealing the solution directly.
    \medskip
    
    TASK:
    
    Given a question and its solution, generate 3 levels of hints that progressively guide the student toward solving the problem independently.
    
    \medskip
    
    HINT LEVELS:
    
    - Level 1: Minimal hint - Points to the key concept or approach without specifics
    
    - Level 2: Medium hint - Provides more direction on the method or intermediate steps
    
    - Level 3: Detailed hint - Gives substantial guidance while still requiring the student to complete the solution
    
    \medskip
    
    GUIDELINES:
    
    - Never reveal the final answer
    
    - Each level should be built on the previous one
    
    - Hints should inspire problem-solving, not just provide steps to copy
    
    - Tailor hint difficulty to bridge the gap between the student's level and the solution
    
    \medskip
    
    OUTPUT FORMAT:
    
    \begin{verbatim}
```json
{
  "level_1": "minimal hint text",
  "level_2": "medium hint text",
  "level_3": "detailed hint text"
}
```
    \end{verbatim}
    \end{tcolorbox}
    \label{fig:system_prompt_hint_gen}
\end{figure}

\begin{figure}[H]
    \centering
    \small
    \begin{tcolorbox}[title=User prompt for hint generation, colback=white, colframe=black, boxrule=0.5mm]
    
    Question:
    
    \{problem\}
    
    \medskip
    
    Solution:
    
    \{solution\}
    
    \end{tcolorbox}
    \label{fig:user_prompt_hint_gen}
\end{figure}

\begin{figure}[H]
    \centering
    \small
    \begin{tcolorbox}[title=System prompt for RL, colback=white, colframe=black, boxrule=0.5mm]
    
    Please reason step by step, and put your final answer within \textbackslash\textbackslash boxed\{\}.
    
    \end{tcolorbox}
    \label{fig:system_prompt_rl}
\end{figure}

\begin{figure}[H]
    \centering
    \small
    \begin{tcolorbox}[title=User prompt for RL, colback=white, colframe=black, boxrule=0.5mm]
    
    \{problem\}

    \medskip
    
    Here is a hint to help you:
    
    \{hint\}
    
    \end{tcolorbox}
    \label{fig:user_prompt_rl}
\end{figure}

\section{Detailed implementation settings}
\label{app:detailed_implementation_settings}
\subsection{Training settings}
\paragraph{SFT.} We use OpenRLHF \cite{hu2024openrlhf} for SFT, and set the learning rate as 5e-5, the batch size as 64, warmup ratio as 10\%, and number of epochs as 3. We evaluate on the final checkpoints.

\paragraph{GRPO.} GRPO shares the same training settings as SAGE, as stated in \S\ref{sec:empirical_results}.

\paragraph{LUFFY.} We use the open-source implementation\footnote{https://github.com/ElliottYan/LUFFY} to reproduce LUFFY, and set the batch size as 128, and \textit{ppo\_mini\_batch\_size} as 64. These two hyper-parameters stay the same for all RL methods.

\paragraph{Scaf-GRPO.} We use the open-source implementation\footnote{https://github.com/JIA-Lab-research/Scaf-GRPO} to reproduce Scaf-GRPO.

\subsection{Evaluation settings}
For the main results in Table \ref{tab:main_results}, we evaluate all models with a max response length of 8192, \textit{temperature=0.6} and \textit{top\_p=0.95}. For the rest, we set a max response length of 2048.

\begin{table*}[h!]
    \caption{Detailed number for Figure \ref{fig:poc_hints} (Left). By default, we train for 200 steps, but for 400 steps for methods denoted by $^*$. The results for level $l=2$ in Figure \ref{fig:poc_hints} are from methods denoted by $^*$. Different from the default training setting of SAGE, we set \textit{ppo\_mini\_batch\_size=32} here.}
    \label{tab:detailed_number_for_hint_levels}
    \vskip 0.1in
    \small
    \begin{center}
    \begin{tabular}{lcccccccc}
    \toprule
        \textbf{Method} & \textbf{Hint level} & \textbf{AIME24} & \textbf{AIME25} & \textbf{AMC23} & \textbf{MATH-500} & \textbf{Minerva} & \textbf{Olympiad} & \textbf{Avg.} \\
    \midrule
        Qwen3-4B-Instruct & - & 30.8 & 28.5 & 75.5 & 87.8 & 42.8 & 54.3 & 53.3 \\
        No hint & 0 & 34.0 & 29.2 & 78.8 & 88.6 & 40.7 & 53.1 & 54.1 \\
    \cmidrule(lr){1-9}
        GPT-5 hint & 1 & 34.0 & 29.8 & 83.0 & 89.6 & 42.2 & 56.9 & 55.9 \\
        Self-hint offline & 1 & 35.8 & 28.3 & 80.2 & 89.9 & 42.4 & 56.9 & 55.6  \\
        Self-hint online & 1 & 35.0 & 28.5 & 82.7 & 90.0 & 42.7 & 61.0 & 56.7 \\
    \cmidrule(lr){1-9}
        GPT-5 hint & 2 & 34.8 & 29.6 & 81.1 & 89.6 & 42.9 & 60.9 & 56.5 \\
        Self-hint offline & 2 & 33.1 & 27.3 & 78.1 & 89.0 & 42.5 & 59.2 & 54.9 \\
        Self-hint online & 2 & 34.2 & 27.7 & 81.7 & 89.2 & 43.5 & 61.8 & 56.4 \\
    \cmidrule(lr){1-9}
        GPT-5 hint$^{*}$ & 2 & 34.8 & 29.6 & 81.1 & 89.6 & 42.9 & 60.9 & 56.5 \\
        Self-hint offline$^{*}$ & 2 & 34.0 & 26.0 & 79.8 & 88.5 & 43.7 & 60.4 & 55.4 \\
        Self-hint online$^{*}$ & 2 & 38.5 & 30.8 & 82.7 & 90.3 & 44.4 & 62.9 & 58.3 \\
    \cmidrule(lr){1-9}
        GPT-5 hint & 3 & 37.9 & 26.9 & 78.3 & 89.2 & 43.1 & 62.1 & 56.3 \\
        Self-hint offline & 3 & 36.7 & 27.5 & 83.1 & 89.6 & 43.3 & 61.5 & 57.0 \\
        Self-hint online & 3 &  36.7 & 27.2 & 83.1 & 89.9 & 43.2 & 62.6 & 57.1 \\
    \bottomrule
    \end{tabular}
    \end{center}
    \vskip -0.1in
\end{table*}

%%%%%%%%%%%%%%%%%%%%%%%%%%%%%%%%%%%%%%%%%%%%%%%%%%%%%%%%%%%%%%%%%%%%%%%%%%%%%%%
\section{Proofs}
\label{app:proofs}

This appendix provides detailed proofs for the results in
Section~\ref{sec:analysis_sage} (and additional analysis).

\subsection{Preliminaries: Bernoulli groups and sample variance}
\label{app:prelim}

Fix a context $(x,h)$ and draw $G\ge2$ i.i.d.\ rollouts with terminal rewards
$R_i \in \{0,1\}$.
Let
\[
\bar R \coloneqq \frac{1}{G}\sum_{i=1}^G R_i,
\qquad
s^2 \coloneqq \frac{1}{G}\sum_{i=1}^G (R_i-\bar R)^2,
\qquad
s \coloneqq \sqrt{s^2}.
\]
When standardized GRPO is used, advantages are
\[
A_i \coloneqq \frac{R_i-\bar R}{s+\epsilon},
\]
where $\epsilon>0$ is a numerical stabilizer (as used in the main text).

We will repeatedly use the fact that since $R_i\in\{0,1\}$,
\begin{equation}
\label{eq:binary_sum_identity_app}
\sum_{i=1}^G (R_i-\bar R)^2
=
\sum_{i=1}^G R_i^2 - G\bar R^2
=
\sum_{i=1}^G R_i - G\bar R^2
=
G\bar R - G\bar R^2
=
G\bar R(1-\bar R),
\end{equation}
hence
\begin{equation}
\label{eq:s2_binary_app}
s^2=\bar R(1-\bar R).
\end{equation}
In particular, $s^2=0$ iff $\bar R\in\{0,1\}$, i.e., iff all $R_i$ are identical.

%%%%%%%%%%%%%%%%%%%%%%%%%%%%%%%%%%%%%%%%%%%%%%%%%%%%%%%%%%%%%%%%%%%%%%%%%%%%%%%
\subsection{Proof of Corollary~\ref{corl:gate_energy}}
\label{app:proof_gate_energy}

\begin{proof}
Recall the advantage energy
\[
E \coloneqq \frac{1}{G}\sum_{i=1}^G A_i^2,
\qquad
A_i=\frac{R_i-\bar R}{s+\epsilon}.
\]
For $\epsilon>0$, we compute
\[
E
=
\frac{1}{G}\sum_{i=1}^G
\frac{(R_i-\bar R)^2}{(s+\epsilon)^2}
=
\frac{1}{(s+\epsilon)^2}\cdot
\frac{1}{G}\sum_{i=1}^G (R_i-\bar R)^2
=
\frac{s^2}{(s+\epsilon)^2},
\]
which is exactly Eq.~\eqref{eq:energy_eps}.

Since $s\ge0$ and $\epsilon>0$, we have $0\le \frac{s}{s+\epsilon}<1$, hence
\[
0\le E=\left(\frac{s}{s+\epsilon}\right)^2<1,
\]
so $E\in[0,1)$ (and in the limit $\epsilon\to0^+$, $E\to \mathbb{I}[s>0]$).

Monotonicity in $s$ follows by differentiation: define
$f(s)\coloneqq \frac{s^2}{(s+\epsilon)^2}$ for $s\ge0$.
Then
\[
f'(s)=\frac{2s(s+\epsilon)^2 - s^2\cdot 2(s+\epsilon)}{(s+\epsilon)^4}
=\frac{2s\epsilon}{(s+\epsilon)^3}\ge0,
\]
so $E$ is non-decreasing in $s$.

Finally, if the group is degenerate, then $s=0$ and therefore $E=0$.
This shows the standardized signal energy collapses to $0$ exactly when within-group variance collapses.
\end{proof}

%%%%%%%%%%%%%%%%%%%%%%%%%%%%%%%%%%%%%%%%%%%%%%%%%%%%%%%%%%%%%%%%%%%%%%%%%%%%%%%
\subsection{Proof of Proposition~\ref{cor:gate_prob}}
\label{app:proof_gate_prob}

\begin{proof}
Write $p\coloneqq p_\theta(x,h)=\Pr[R(x,\tau)=1\mid x,h]$.
Then $R_1,\dots,R_G$ are i.i.d.\ Bernoulli$(p)$.

As noted in Section~\ref{app:prelim}, $s=0$ iff all rewards are identical.
There are exactly two degenerate cases:
(i) all-zero: $R_1=\cdots=R_G=0$; (ii) all-one: $R_1=\cdots=R_G=1$.
Thus
\[
\Pr[s>0\mid x,h]
=
1-\Pr[\text{all-zero}]-\Pr[\text{all-one}]
=
1-(1-p)^G - p^G,
\]
which proves Eq.~\eqref{eq:gate_prob}.

To locate the maximizer, define $u(p)\coloneqq 1-(1-p)^G-p^G$ on $[0,1]$.
We have symmetry $u(p)=u(1-p)$.
Moreover, for $G\ge2$,
\[
u''(p)
=
-\;G(G-1)\Big(p^{G-2}+(1-p)^{G-2}\Big) < 0,
\]
so $u$ is strictly concave, hence has a unique maximizer.
By symmetry, the unique maximizer must be at $p=\tfrac12$.

Finally, in the sparse regime $p\ll1$,
\[
u(p)=1-(1-p)^G-p^G
=
1-\Big(1-Gp+O(p^2)\Big)-O(p^G)
=
Gp+O(p^2),
\]
so $\Pr[s>0\mid x,h]\approx Gp$.
\end{proof}

%%%%%%%%%%%%%%%%%%%%%%%%%%%%%%%%%%%%%%%%%%%%%%%%%%%%%%%%%%%%%%%%%%%%%%%%%%%%%%%
\subsection{Proof of Proposition~\ref{prop:q_dep_theta}}
\label{app:proof_q_dep_theta}

\begin{proof}
Fix $x$ and $G\ge2$.
Recall $u(p)=1-(1-p)^G-p^G$ and
\[
J_x(\theta,q)
=
\mathbb{E}_{h\sim q(\cdot\mid x)}\Big[u\big(p_\theta(x,h)\big)\Big].
\]

We first verify the claims about $u$.
Symmetry: $u(1-p)=1-p^G-(1-p)^G=u(p)$.
For strict concavity, compute for $G\ge2$:
\[
u''(p)
=
-\;G(G-1)\Big(p^{G-2}+(1-p)^{G-2}\Big) < 0,
\]
so $u$ is strictly concave on $[0,1]$.
By symmetry and strict concavity, the unique maximizer is $p=\tfrac12$.

For a fixed $\theta$, define the measurable function
\[
v_\theta(h)\coloneqq u\!\left(p_\theta(x,h)\right)\in[0,1].
\]
Then
$
J_x(\theta,q)=\mathbb{E}_{h\sim q}[v_\theta(h)].
$
Over all probability distributions $q(\cdot\mid x)$ supported on the admissible hint space,
the maximizers must concentrate probability mass on (essential) maximizers of $v_\theta$:
indeed, since the objective is linear in $q$, any optimizer can be chosen to put all mass on
\[
\arg\max_h v_\theta(h)=\arg\max_h u\!\left(p_\theta(x,h)\right).
\]
Because $u$ is uniquely maximized at $p=\tfrac12$ and is strictly decreasing as $p$ moves away from $\tfrac12$
(due to strict concavity and symmetry), the maximizers of $v_\theta(h)$ are exactly the
\emph{calibrating hints} that make $p_\theta(x,h)$ as close as possible to $\tfrac12$
(and equal to $\tfrac12$ when achievable).
This proves the statement that an optimal $q_\theta^\star(\cdot\mid x)$ places its mass on
calibrating hints.

In general, $p_\theta(x,h)$ changes with $\theta$ because the rollout distribution under
$\pi_\theta(\cdot\mid x,h)$ changes with $\theta$.
Therefore the set of calibrating hints
\[
\mathcal{H}_\theta^\star(x)
\coloneqq
\arg\max_h u\!\left(p_\theta(x,h)\right)
\]
typically varies with $\theta$.
Unless $p_\theta(x,h)$ (hence $v_\theta(h)$) is invariant in $\theta$ for $q$-almost all $h$,
a fixed distribution $q$ that was optimal (or near-optimal) early in training will drift away from being optimal later,
reducing $J_x(\theta,q)$ relative to a $\theta$-adapted choice.
This formalizes why updating the hint generator online using the policy can reduce gate-mismatch.
\end{proof}

%%%%%%%%%%%%%%%%%%%%%%%%%%%%%%%%%%%%%%%%%%%%%%%%%%%%%%%%%%%%%%%%%%%%%%%%%%%%%%%
\subsection{Proof of the Jensen inequality in Remark~1}
\label{app:proof_jensen}

\begin{proof}
Fix $G\ge2$ and define $u(p)=1-(1-p)^G-p^G$.
We showed above that $u$ is concave on $[0,1]$ because $u''(p)\le0$.
Let $Z\coloneqq p_\theta(x,h)\in[0,1]$ be the random success probability induced by sampling $h\sim q$.
Then Jensen's inequality for concave $u$ gives
\[
\mathbb{E}_h\big[u(Z)\big]
\le
u\!\left(\mathbb{E}_h[Z]\right),
\]
which is exactly Eq.~\eqref{eq:jensen_gate}.
Equality holds only when $Z$ is almost surely constant (or when $u$ is affine on the support,
which does not happen for $G\ge2$ except in degenerate cases).
This shows that, at a fixed mean success probability, additional variability across hints can only decrease
the expected gate-opening frequency.
\end{proof}

%%%%%%%%%%%%%%%%%%%%%%%%%%%%%%%%%%%%%%%%%%%%%%%%%%%%%%%%%%%%%%%%%%%%%%%%%%%%%%%
\subsection{A sharper small-$p$ expansion of the gate probability}
\label{app:smallp}

For completeness, we also record an exact decomposition that makes the $Gp$ scaling explicit.
Let $p=\Pr[R=1\mid x,h]$.
Then
\begin{align}
\Pr[s>0\mid x,h]
&=
1-(1-p)^G-p^G \nonumber\\
&=
\sum_{k=1}^{G-1}\binom{G}{k}p^k(1-p)^{G-k}.
\label{eq:gate_binomial_decomp}
\end{align}
When $p\ll1$, the dominant term is $k=1$:
\[
\Pr[s>0\mid x,h]
=
G p (1-p)^{G-1} + O(p^2)
\approx Gp,
\]
and the neglected $p^G$ term is exponentially smaller in $G$.

%%%%%%%%%%%%%%%%%%%%%%%%%%%%%%%%%%%%%%%%%%%%%%%%%%%%%%%%%%%%%%%%%%%%%%%%%%%%%%%
\subsection{Non-standardized GRPO signal energy}
\label{app:nonstd_energy}

Non-standardized advantage is also widely used in GRPO-like algorithms \citep{liu2025understanding}, which also provide insights about the behavior.

\paragraph{Setup.}
Define non-standardized (mean-centered) advantages
\[
\tilde A_i\coloneqq R_i-\bar R,
\qquad
\bar R=\frac{1}{G}\sum_{i=1}^G R_i,
\]
and define the (non-standardized) energy
\[
\tilde E \coloneqq \frac{1}{G}\sum_{i=1}^G \tilde A_i^2.
\]
Note that $\tilde E=s^2$ by definition.

\begin{proposition}[Expected non-standardized energy under Bernoulli rewards]
\label{prop:nonstd_energy}
Conditioned on $(x,h)$ with success probability $p=p_\theta(x,h)$,
\begin{equation}
\mathbb{E}\big[\tilde E \mid x,h\big]
=
\frac{G-1}{G}\,p(1-p).
\label{eq:nonstd_energy}
\end{equation}
\end{proposition}

\begin{proof}
Let $S=\sum_{i=1}^G R_i$ so that $\bar R=S/G$.
Using identity~\eqref{eq:binary_sum_identity_app},
\[
\tilde E=\frac{1}{G}\sum_{i=1}^G (R_i-\bar R)^2
=
\bar R(1-\bar R).
\]
Taking expectation:
\[
\mathbb{E}[\tilde E]
=
\mathbb{E}[\bar R]-\mathbb{E}[\bar R^2].
\]
We have $\mathbb{E}[\bar R]=p$.
Also,
\[
\mathbb{E}[\bar R^2]
=
\mathrm{Var}(\bar R)+(\mathbb{E}[\bar R])^2
=
\frac{1}{G^2}\mathrm{Var}(S)+p^2
=
\frac{1}{G^2}\cdot Gp(1-p)+p^2
=
\frac{p(1-p)}{G}+p^2.
\]
Therefore
\[
\mathbb{E}[\tilde E]
=
p-\left(\frac{p(1-p)}{G}+p^2\right)
=
\frac{G-1}{G}\,p(1-p).
\]
\end{proof}

\begin{proposition}[Optimal calibrated difficulty for mean-centered updates]
\label{prop:nonstd_opt}
For fixed $G\ge2$, the right-hand side of Eq.~\eqref{eq:nonstd_energy} is uniquely maximized at $p=\tfrac12$.
\end{proposition}

\begin{proof}
Let $f(p)=p(1-p)=p-p^2$.
Then $f'(p)=1-2p$ and $f''(p)=-2<0$, so $f$ is strictly concave and uniquely maximized at $p=\tfrac12$.
The prefactor $(G-1)/G$ does not affect the maximizer.
\end{proof}

\begin{remark}[Hint randomness incurs a variance penalty at fixed mean]
\label{rmk:var_penalty}
Let $Z=p_\theta(x,h)$ be random due to $h\sim q$ and $\bar p=\mathbb{E}[Z]$.
Then
\[
\mathbb{E}\big[Z(1-Z)\big]=\bar p(1-\bar p)-\mathrm{Var}(Z).
\]
Thus, at fixed mean success rate, additional variability in $p_\theta(x,h)$ across hints reduces the expected energy.
\end{remark}

\begin{proof}
Compute
\[
\mathbb{E}[Z(1-Z)]
=
\mathbb{E}[Z]-\mathbb{E}[Z^2]
=
\bar p-\left(\mathrm{Var}(Z)+\bar p^2\right)
=
\bar p(1-\bar p)-\mathrm{Var}(Z).
\]
\end{proof}

\end{document}